\documentclass{article}

\usepackage[utf8]{inputenc} % allow utf-8 input
\usepackage[T1]{fontenc}    % use 8-bit T1 fonts
\usepackage{hyperref}       % hyperlinks
\usepackage{url}            % simple URL typesetting
\usepackage{booktabs}       % professional-quality tables
\usepackage{amsfonts}       % blackboard math symbols
\usepackage{nicefrac}       % compact symbols for 1/2, etc.
\usepackage{microtype}      % microtypography
\usepackage{xcolor}         % colors
\title{Heavy-tailed Streaming Statistical Estimation}

\author{%
  Che-Ping Tsai$^\dagger$, \ Adarsh Prasad$^\dagger$,  Sivaraman Balakrishnan$^\ddagger$, \ Pradeep Ravikumar$^\dagger$ \\
  \\
  \texttt{chepingt@cs.cmu.edu, adarshp@andrew.cmu.edu} \\
  \texttt{siva@stat.cmu.edu, pradeepr@cs.cmu.edu} \\
  \\
  $^\dagger$Department of Machine Learning\\
  $^\ddagger$Department of Statistics and Data Science\\
  Carnegie-Mellon University\\
  Pittsburgh, PA 15213 \\
}

\usepackage[numbers]{natbib}
\usepackage{parskip}

\usepackage{graphicx}
\usepackage{amsmath,amssymb,amsthm} 
\usepackage{fullpage}
\usepackage[T1]{fontenc} 
\usepackage{epsf} 
\usepackage{graphics} 
\usepackage{color,etoolbox}
\usepackage{bbm}
\usepackage{mysymbols}
\usepackage{comment} 
\usepackage{pifont}
\usepackage{subfigure}
\usepackage{multirow}
\usepackage{multicol}
\usepackage{wrapfig}
\usepackage{color, colortbl}
\usepackage{float}

\usepackage{algorithm}
\usepackage[noend]{algpseudocode}

\newcommand{\loss}{\overline{\calL}}
\newcommand{\clip}{\mathrm{clip}}

\newcommand{\z}{x}
\newcommand{\circledOne}{\text{\ding{172}}}
\newcommand{\circledTwo}{\text{\ding{173}}}
\newcommand{\circledThree}{\text{\ding{174}}}
\newcommand{\circledFour}{\text{\ding{175}}}
\newcommand{\circledFive}{\text{\ding{176}}}
\newcommand{\circledSix}{\text{\ding{177}}}

\begin{document}


\maketitle

\begin{abstract}
We consider the task of heavy-tailed statistical estimation given streaming $p$-dimensional samples. This could also be viewed as stochastic optimization under heavy-tailed distributions, with an additional $O(p)$ space complexity constraint. We design a clipped stochastic gradient descent algorithm and provide an improved analysis, under a more nuanced condition on the noise of the stochastic gradients, which we show is critical when analyzing stochastic optimization problems arising from general statistical estimation problems. Our results guarantee convergence not just in expectation but with exponential concentration, and moreover does so using $O(1)$ batch size. We provide consequences of our results for mean estimation and linear regression. Finally, we provide empirical corroboration of our results and algorithms via synthetic experiments for mean estimation and linear regression.
\end{abstract}

%\section{INTRODUCTION}
\section{Introduction}

Statistical estimators are typically random, since they depend on a random training set; their statistical guarantees are typically stated in terms of the expected loss between estimated and true parameters~\citep{kingma2014adam,duchi2011adaptive,zhang2019adaptive,hazan2014beyond}. A bound on expected loss however might not be sufficient in higher stakes settings, such as autonomous driving, and risk-laden health care, among others, since the deviation of the estimator from its expected behavior could be large. In such settings, we might instead prefer a bound on the loss that holds with high probability. Such high-probability bounds are however often stated only under strong assumptions (e.g. sub-Gaussianity or boundedness) on the tail of underlying distributions~\citep{harvey2019simple,harvey2019tight,rakhlin2011making,ghadimi2013optimal}; conditions which often do not hold in real-world settings. There has also been a burgeoning line of recent work that relaxes these assumptions and allows for heavy-tailed underlying distributions~\citep{catoni2012challenging,hsu2016loss,lugosi2019mean}, but the resulting algorithms are often not only complex, but are also specifically batch learning algorithms that require storing the entire dataset, which limits their scalability. For instance, many popular polynomial time algorithms on heavy-tailed mean estimation~\citep{diakonikolas2017being,cheng2020high,cherapanamjeri2019fast,lei2020fast,dong2019quantum,depersin2019robust} and heavy-tailed linear regression algorithms~\citep{hsu2016loss,suggala2019adaptive,pensia2020robust} need to store the dataset to take polylogarithmic passes over data.

\begin{wrapfigure}{R}{0.4\textwidth}
\vspace{-0.2cm}
\label{fig:intro_heavy_tail}
\begin{center}
  \includegraphics[width=0.4\textwidth]{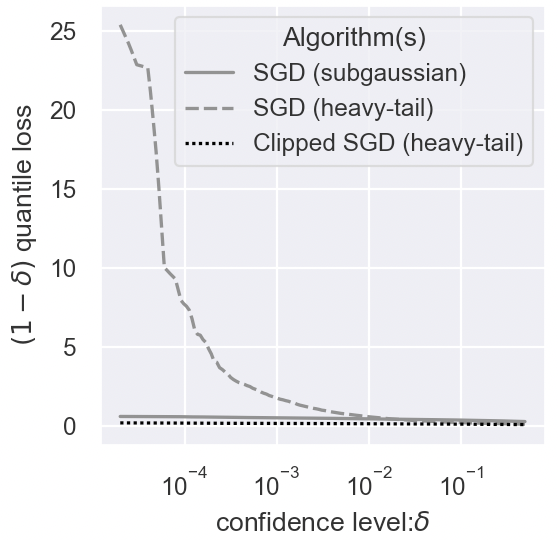}
  \end{center}
  \caption{ Tail performance of SGD and clipped-SGD of mean estimation. The underlying distributions are zero-mean sub-Gaussian or heavy-tailed distributions. This figure shows
  $\ell_2$-loss between estimated and true mean against different confidence levels $\delta$. See more details in Section \ref{sec:details_for_intro_exp}.}
 \vspace{-0.2cm}
\end{wrapfigure}

On the other hand,
most successful practical modern learning algorithms are iterative, light-weight and access data in a ``streaming'' fashion. As a consequence, we focus on designing and analyzing iterative statistical estimators which only use constant storage in each step. To summarize, motivated by practical considerations, we have three desiderata: (1) allowing for heavy-tailed underlying distributions (weak modeling assumptions), (2) high probability bounds on the loss between estimated and true parameters instead of just its expectation (strong statistical guarantees), and (3) estimators that access data in a streaming fashion while only using constant storage (scalable, simple algorithms).

A useful alternative viewpoint of the statistical estimation problem above is that of stochastic optimization: where we have access to the optimization objective function (which in the statistical estimation case is simply the population risk of the estimator) only via samples of the objective function or its higher order derivatives (typically just the gradient). Here again, most of the literature on stochastic optimization typically provides bounds in expectation~\citep{hazan2014beyond,zhang2019adaptive,gower2019sgd}, or places a strong assumptions on the tail behavior of the distributions of the derivatives of the stochastic objective, such as the distributions being bounded~\citep{harvey2019tight,rakhlin2011making} or sub-Gaussian~\citep{harvey2019simple,li2019convergence}. Figure \ref{fig:intro_heavy_tail} shows that even for the simple stochastic optimization task of mean estimation, the deviation of stochastic gradient descent (SGD) is much worse for heavy-tailed distributions than sub-Gaussian ones. Therefore, bounds on expected behavior, or strong assumptions on the tails of the stochastic noise distribution are no longer sufficient.

%These estimators are often less robust and have been shown to have poor performance if these "thin-tailed" assumptions are no longer held \cite{catoni2012challenging,lugosi2019mean}.
%Existing estimators and theories in machine learning are largely designed for "thin-tailed" data, such as those coming from sub-Gaussian distributions. For instance, most of the literature of stochastic optimization either only provided bounds in expectation \cite{hazan2014beyond,zhang2019adaptive,gower2019sgd} or placed a strong assumption on the tail of an underlying distribution, such as bounded distributions \cite{harvey2019tight,rakhlin2011making} and sub-Gaussian distributions \cite{harvey2019simple,li2019convergence}. These estimators are often less robust and have been shown to have  poor performance if these "thin-tailed" assumptions are no longer held \cite{catoni2012challenging,lugosi2019mean}. Moreover, from a practical standpoint, this is a less than desirable state of affairs: it has been shown that heavy-tailed distributions are ubiquitous in a variety of fields including neural networks for image classification \cite{simsekli2019tail}, attention models for natural language processing \cite{zhang2019adaptive}, and reinforcement learning \cite{garg2021proximal}. 

While there has been a line of work on heavy-tailed stochastic optimization, these require non-trivial storage complexity or batch sizes, making them unsuitable for streaming settings or large-scale problems~\citep{fritsch2015robust,mcmahan2013ad}. Specifically these existing works require at least $O(1/\epsilon)$ batch size to obtain a $\epsilon$-approximate solution under heavy-tailed noise~\citep{prasad2018robust,davis2021low,gorbunov2020stochastic,nazin2019algorithms} (See Section \ref{sec:detailed_related_work} in the Appendix for further discussion). In other words, to achieve a typical $O(1/N)$ convergence rate (on the squared error), where $N$ is the number of samples, they would need a batch-size of nearly the entire dataset.

Therefore, we investigate the following question:
%
%Although past works \cite{prasad2018robust,davis2021low,gorbunov2020stochastic,nazin2019algorithms} have proposed algorithms with high-confidence guarantees in stochastic optimization under heavy-tailed noise, their batch sizes are either non-trivial or even exponentially growing with number of iterations (See details in Table \ref{tab:sample_complexity_comparison}). Their algorithms are unrealistic and questionable to be used in streaming setting or large-scale problems \citep{fritsch2015robust,mcmahan2013ad}, for which the stochastic optimization algorithms are originally designed. Therefore, we investigate the following question in stochastic optimization: 
\begin{center}
\emph{Can we develop a stochastic optimization method that satisfy our three desiderata?}
%: comes with bounds that hold with high-probability, under heavy-tailed stochastic noise distributions, and in a streaming setting with constant storage complexity? }
\end{center}
Our answer is that a simple algorithm suffices: \emph{stochastic gradient descent with clipping} (\texttt{clipped-SGD}). 
In particular, we first prove a high probability bound for \texttt{clipped-SGD} under heavy-tailed noise, with a decaying $O(1/t)$ step-size sequence for strongly convex objectives. By using a decaying step size, we improve the analysis of~\citep{gorbunov2020stochastic} and develop the first robust stochastic optimization algorithm \emph{in a fully streaming setting} - i.e. with $O(1)$ batch size. We then consider corollaries for statistical estimation where the optimization objective is the population risk of the estimator, and derive the first streaming robust mean estimation and linear regression algorithm that satisfy all three desiderata above.

We summarize our contributions as follows:

\begin{itemize}

    \item We prove the first high-probability bounds for \texttt{clipped-SGD} with a $O(1/t)$ step size and a constant batch size for strongly convex and smooth objectives \emph{without sub-Gaussian assumption on stochastic gradients}. To the best of our knowledge, this is the first stochastic optimization algorithm that uses constant batch size in this setting. See Section \ref{sec:related_work} and Table \ref{tab:sample_complexity_comparison} for more details and comparisons. A critical ingredient is a nuanced condition on the stochastic gradient noise.
    
    \item We show that our proposed stochastic optimization algorithm can be used for a broad class of statistical estimation problems. As corollaries of this framework, we present a new class of robust heavy-tailed estimators for streaming mean estimation and linear regression.
    
    \item Lastly, we conduct synthetic experiments to corroborate our theoretical and methodological developments. We show that \texttt{clipped-SGD} not only outperforms SGD and a number of baselines in average performance but also has a well-controlled tail performance.
 
\end{itemize}

\subsection{Related Work}
\label{sec:related_work}

\paragraph{Batch Heavy-tailed mean estimation.} In the batch-setting,
~\cite{hopkins2020mean} proposed the first polynomial-time algorithm that matches the error guarantees achieved by the empirical mean on Gaussian data. After this work, efficient algorithms with improved asymptotic runtimes were proposed: 
% SDP :hopkins2020mean, cherapanamjeri2019fast
% spectral techniques:diakonikolas2019robust,hopkins2020robust, cheng2020high, lei2020fast
% Median-of-means:depersin2019robust (not poly time ), minsker2015geometric(sub-optimal)
\citet{hopkins2020mean, cherapanamjeri2019fast} proposed optimal estimators based on semi-definite programming (SDP). \citet{diakonikolas2019robust,hopkins2020robust, cheng2020high, lei2020fast, dong2019quantum} constructed more practical algorithms via spectral techniques. Estimators~\citep{minsker2015geometric,depersin2019robust} relied on the median-of-means framework.
However, these approaches are not designed for the streaming setting and requires taking polylogarithmic passes over data. We discuss the difficulties in applying these approaches in the streaming setting in Section \ref{sec:cor_mean_estimation}. 
%Some exceptions are ~\citet{minsker2015geometric,depersin2019robust}, which have $O(\log(1/\delta))$ storage complexity to obtain a solution with a $\delta$ failure probability. However, 
%these works rely on the median-of-means framework and are not easy to extend to the streaming setting (See Section \ref{sec:cor_mean_estimation} for more information). 
%For more details on these recent algorithms, see the recent survey~\citep{diakonikolas2019recent}. 

\paragraph{Batch Heavy-tailed regression.}
For the setting where the regression noise $w$ is heavy-tailed with bounded variance, Huber's estimator is known to have exponential deviation bounds~\citep{fan2017estimation,sun2020adaptive} in high dimensional setting. For the case where both the covariates and the noise are both heavy-tailed, several recent works have proposed computationally efficient estimators that achieve exponential deviation bounds based on the median-of-means framework~\citep{lugosi2019risk, hsu2016loss,minsker2015geometric}, thresholding techniques~\citep{suggala2019adaptive}, and covariate filtering \citep{pensia2020robust}.
However, as we noted before, computing all of these estimators require storing the entire dataset.
%Also, other recent work  proposed estimators based on Catoni's mean estimator \cite{brownlees2015empirical} and median-of-mean tournament \cite{lugosi2019risk} with strong theoretical guarantees but their estimators are not easily computable. 

\paragraph{Heavy-tailed stochastic optimization.}
A line of work in stochastic convex optimization have proposed bounds that achieve sub-Gaussian concentration around their mean (a common step towards providing sharp high-probability bounds), while only assuming that the variance of stochastic gradients is bounded (i.e. allowing for heavy-tailed stochastic gradients). \citet{davis2021low} proposed \texttt{proxBoost} that is based on robust distance estimation and proximal operators. \citet{prasad2018robust} utilized the geometric median-of-means to robustly estimate gradients in each mini-batch. \citet{gorbunov2020stochastic} and \citet{nazin2019algorithms} proposed \texttt{clipped-SSTM} and \texttt{RSMD} respectively based on truncation of stochastic gradients for stochastic mirror/gradient descent. \citet{zhang2019adaptive} analyzed the convergence of \texttt{clipped-SGD} \emph{in expectation} but focus on a different noise regime where the distribution of stochastic gradients has bounded $1+\alpha$ moments for some $0 < \alpha \leq 1$. However, all the above works ~\citep{davis2021low,prasad2018robust,gorbunov2020stochastic,nazin2019algorithms} have an unfavorable $O(n)$ dependency on the batch size to get the typical $O(1/n)$ convergence rate (on the squared error). We note that our bound is comparable to the above approaches while using a constant batch size. See Appendix \ref{sec:detailed_related_work} for more details.  
%\section{BACKGROUND AND PROBLEM FORMULATION }
\section{Background and Problem Formulation}

In this paper, we consider the following statistical estimation/stochastic optimization setting: we assume that there are some class of functions $\{f_\theta \}_{\theta \in \Theta}$ parameterized by $\theta$, where $\Theta$ is a convex subset of $\real^p$; some random vector $\z$ with distribution $P$; and a loss function $\loss$ which takes $\z, \theta$ and outputs the loss of $f_\theta$ at point $\z$. In this setting, we want to recover the true parameter $\theta^*$ defined as the minimizer of the population risk function $\calR(\theta)$:
\begin{equation}
\label{eqn:risk_minimization}
\theta^*
= \argmin_{\theta \in \Theta } \calR(\theta) 
= \argmin_{\theta \in \Theta} \Exp_{\z \sim P} [\loss(\theta, \z)].
\end{equation}
We assume that $\loss$ is differentiable and convex, and further impose two regularity conditions on the population risk: there exist $\tau_\ell$ and $\tau_u$ such that 
\begin{align}
\label{eqn:convexity_parameter}
\frac{\tau_\ell}{2} \norm{\theta_1 - \theta_2}{2}^2 
\leq \calR(\theta_1) - \calR(\theta_2) - \inprod{\nabla \calR(\theta_2)}{\theta_1 - \theta_2} \leq \frac{\tau_u}{2} \norm{\theta_1 - \theta_2}{2}^2 
\end{align}
with $\tau_\ell, \tau_u > 0$ and it holds for all $\theta_1, \theta_2 \in \Theta$.  The parameters $\tau_\ell, \tau_u$ are called the strong-convexity and smoothness parameters of the function $\calR(\theta)$.

To solve the minimization problem defined in Eq. \eqref{eqn:risk_minimization}, we assume that we can access the stochastic gradient of the population risk, $\nabla \loss(\theta,\z)$, at any point $\theta \in \Theta$ given a sample $\z$. We note that this is a unbiased gradient estimator, i.e.
$$
\Exp_{\z \sim P }[\nabla \loss(\theta,\z)] = \nabla \calR(\theta).
$$
Our goal in this work is to develop robust statistical methods under heavy-tailed distributions. The specific characterization of heavy-tailed distributions we consider in this paper is the common notion of distributions where only very low order moments may be finite, e.g. student-t distribution or Pareto distribution. In this work, we assume the stochastic gradient distribution only has bounded second moment. Formally, for any $\theta \in \Theta$, we assume that there exists $\alpha(P, \loss)$ and $\beta(P, \loss)$ such that 
\begin{align}
\label{eqn:asp_variance_sg}
\Exp_{\z \sim P} [\norm{ & \nabla \loss(\theta,\z)  - \nabla \calR(\theta)}{2}^2] 
\leq \alpha(P, \loss)\norm{\theta-\theta^*}{2}^2 + \beta(P, \loss).
\end{align}
In other words, the variance of the $\ell_2$-norm of the gradient distribution depends on a uniform constant $ \beta(P, \loss)$ and a position-dependent variable, $\alpha(P,\loss)$, which allows the variance of gradient noise to be large when $\theta$ is far from the true parameter $\theta^*$. We note that this is a more general assumption compared to prior works which assumed that the variance is uniformly bounded by $\sigma^2$. It can be seen that our condition is more nuanced and can be weaker: even if our condition holds, we would allow for a uniform bound on variance to be large: $\alpha(P,\loss) \sup_{\theta \in \Theta} \|\theta - \theta^*\|^2 + \beta(P, \loss)$. Whereas, a uniform bound on the variance could always be cast as $\alpha(P,\loss) = 0$ and $\beta(P,\loss) = \sigma^2$. We will show that this more nuanced assumption is essential to obtain tight bounds for linear regression problems. 

We next provide some running examples to instantiate the above:

\begin{enumerate}
    \item \textbf{Mean estimation:} Given observations $\z_1, \cdots, \z_n \sim P$ where the distribution $P$ with mean $\mu$ and a bounded covariance matrix $\Sigma$. The minimizer of the following square loss is the mean $\mu$ of distribution $P$:
    \begin{equation}
    \label{eqn:def_mean_estimation}
    \loss(\theta, \z ) = \frac{1}{2}\norm{\z - \theta}{2}^2
    \ \text{ and }  \ 
    \mu = \argmin_{\theta \in \real^p} \Exp_{\z \sim P}  [ \loss(\theta, \z )].
    \end{equation}
    In this case, $\tau_\ell = \tau_u = 1$, $\alpha(P,\loss) = 0$ and $\beta(P,\loss) = \trace{\Sigma}$ satisfy the assumption in Eq.\eqref{eqn:asp_variance_sg}.
    
    \item \textbf{Linear regression:} Given covariate-response pairs $(x,y)$, where $x, y$ are sampled from $P$ and have a linear relationship, i.e.  $y = \inprod{x}{\theta^*} + w$, where $\theta^*$ is the true parameter we want to estimate and $w$ is drawn from a zero-mean distribution. Suppose that under distribution $P$ the covariate $x \in \real^p$ have mean $0$ and non-singular covariance matrix $\Sigma$.
    In this setting, we consider the squared loss:
    \begin{align}
    \label{eqn:def_linear_regression}
     \loss(\theta, (x,y)) = \frac{1}{2}( y - \inprod{x}{\theta})^2,
     \text{ and }  \ \
    \calR(\theta) = \frac{1}{2}(\theta -\theta^*)^{\top} \Sigma (\theta - \theta^*).
    \end{align}
    The true parameter $\theta^*$ is the minimizer of $\calR(\theta)$. We also note that $\tau_{\ell} = \lambda_{\min}(\Sigma)$ and $\tau_u = \lambda_{\max} ( \Sigma)$ satisfies the assumption in Eq.\eqref{eqn:convexity_parameter} with  $\alpha(P, \loss) = O(p\norm{\Sigma}{2}^2)$, and $\beta(P,\loss) = p\sigma^2 \norm{\Sigma}{2}$ satisfy the assumption in Eq.\eqref{eqn:asp_variance_sg}. Note that if we had to uniformly bound the variance of the gradients as in previous stochastic optimization work, that bound would need to scale as:
    $O(p\norm{\Sigma}{2}^2R^2 + p\sigma^2 \norm{\Sigma}{2})$, where $R = \sup_{\theta \in \Theta}\|\theta - \theta^*\|$, which will yield much looser bounds.
\end{enumerate}

%\section{MAIN RESULTS}
\section{Main Results}
\label{sec:main_results}

In this section, we introduce our clipped  stochastic gradient descent algorithm. We begin by formally defining clipped stochastic gradients. For a clipping parameter $\lambda \geq 0$:
\begin{equation}
\label{eqn:def_clip_function}
\clip(\nabla \loss(\theta, \z), \lambda) = \min \left(1, \frac{\lambda}{\norm{\nabla \loss(\theta, \z)}{2}} \right)\nabla \loss(\theta, \z),
\end{equation}
where $\theta \in \Theta$, $\z \sim P$ and $\nabla \loss(\theta, \z)$ is the stochastic gradient. The overall algorithm is summarized in Algorithm \ref{algo:clipped_sgd}, where we use $\calP_\Theta$ to denote the (Euclidean) projection onto the domain $\Theta$.

%\begin{wrapfigure}{R}{0.55\textwidth}
%\begin{minipage}{0.55\textwidth}
\begin{algorithm}[H]
\caption{Clipped stochastic gradient descent (\texttt{clipped-SGD}) algorithm.}
%\small
\hspace*{\algorithmicindent} \textbf{Input:} loss function $\loss$, initial point $\theta^1$, step size $\eta_t$, clipping level $\lambda$, samples $\z_1,\cdots,\z_{N} \sim P$.
\begin{algorithmic}[1]
    \For{$t=1,2,...,N$}
    \State $\theta^{t+1} \leftarrow \calP_{\Theta} \left(\theta^{t} -  \eta_t \clip(\nabla \loss(\theta^{t},\z_t), \lambda) \right)$.
    \EndFor
\end{algorithmic}
\hspace*{\algorithmicindent} \textbf{Output:} $\theta^{N+1}$.
\label{algo:clipped_sgd}
\end{algorithm}
%\end{minipage}
%\end{wrapfigure}
Next, we state our main convergence result for \texttt{clipped-SGD} in Theorem \ref{thm:clipped_sgd}.
\begin{theorem}
\label{thm:clipped_sgd}
(Streaming heavy-tailed stochastic optimization)
Suppose that the population risk satisfies the regularity conditions in Eq. \eqref{eqn:convexity_parameter} and stochastic gradient noise satisfies the condition in Eq. \eqref{eqn:asp_variance_sg}. Let $\delta \in (0,2e^{-1})$ and 
\begin{equation}
\gamma \defeq 144\max \left\{\frac{\tau_u}{\tau_\ell}, \frac{96\alpha(P,\loss)}{\tau_\ell^2} \right\} \log(2/\delta) + 1.
\end{equation}
Given $N$ samples $\z_1,\cdots, \z_N$, the Algorithm \ref{algo:clipped_sgd} initialized at $\theta^1$ with
$ \eta_t \defeq \frac{1}{\tau_\ell(t+\gamma)} $ and 
\begin{equation}
\label{eqn:def_hyperparemeter}
\lambda \defeq  C_1\sqrt{\frac{\tau_\ell^2 \gamma(\gamma-1)\norm{\theta^1 - \theta^*}{2}^2}{\log(2/\delta)^2}
+ \frac{(N+\gamma)\beta(P,\loss)}{\log(2/\delta)}},
\end{equation} 
where $C_1 \geq 1$ is a scaling constant can be chosen by users, returns $\theta^{N+1}$ such that with probability at least $1-\delta$, we have
\begin{equation}
\label{eqn:theorem_l2_bound}
\norm{\theta^{N+1} -\theta^*}{2}
\leq 100 C_1\left(\frac{\gamma \norm{\theta^1 - \theta^*}{2}}{N+\gamma}
+ \frac{1}{\tau_\ell}\sqrt{\frac{\beta(P,\loss)\log(2/\delta)}{N+\gamma}} \right).
\end{equation}
\end{theorem}
We explain our theoretical contribution and provide a proof sketch in Section~\ref{sec:sketch_of_proof}.
The complete proof can be found in  Appendix~\ref{sec:proof_clipped_sgd}.

\paragraph{Remarks:}

a) This theorem says that with a properly chosen clipping level $\lambda$, \texttt{clipped-SGD} has an asymptotic convergence rate of $O(1/\sqrt{N})$ and enjoys \textit{sub-Gaussian style concentration} around the true minimizer (alternatively, its high probability bound scales logarithmically in the confidence parameter $\delta$). The first term in the error bound is related to the \emph{initialization error} and the second term is governed by the \emph{stochastic gradient noise}. These two terms have different convergence rates: at early iterations when the initialization error is large, the first term dominates but quickly decreases at the rate of $O(\gamma/N)$. At later iterations the second term dominates and decreases at the usual statistical rate of convergence of $O(1/\sqrt{N})$. 

b) Note that $\eta_t = O(\frac{1}{\tau_\ell t})$ is a common choice for optimizing $\tau_\ell$-strongly convex functions~\citep{harvey2019tight,rakhlin2011making}. The only difference is that we add a delay parameter $\gamma$ to "slow down" the learning rate $\eta_t$ and to stablize the training process. The delay parameter $\gamma$ depends on the position-dependent variance term $\alpha(P,\loss)$ and the condition number $\tau_u/\tau_\ell$. From a theoretical viewpoint, the delay parameter ensures that the true gradient is within the clipping region with high-probability, i.e. $\norm{\nabla \calR(\theta^t)}{2} \leq \lambda$ for $t=1,....,N$ with high probability and this in turn allows us to control the variance and the bias incurred by clipped gradients. Moreover, it controls the  position-dependent variance term $\alpha(P,\loss)\norm{\theta^t -\theta^*}{2}^2$, especially during the initial iterations when the error (and the variance of stochastic gradients) is large. 

%This can be corroborated by empirical observations that using a small $\gamma$ (large step size) leads to divergence at the early stage in linear regression.

c) We choose the clipping level to be proportional to $\sqrt{N}$ to balance the variance and bias of clipped gradients. Roughly speaking, the bias is inversely proportional to the clipping level (Lemma~\ref{lm:noise_term_control}). As the error $\norm{\theta^N-\theta^*}{2}$ converges at the rate $O(1/\sqrt{N})$, and this in turn suggests that we should choose the clipping level to be $O(\sqrt{N})$.

d) Note that previous 
stochastic optimization algorithms use  $O(n)$ batch sizes for strongly convex objective \citep{davis2021low,gorbunov2018accelerated,nazin2019algorithms,prasad2018robust}.
To address this issue, the critical ingredients are the use of \emph{$O(1/t)$ decayed learning rate} and \emph{a delayed parameter $\gamma$} to prevent it from diverging. Also, we explicitly control the variance of the gradient noise by clipping the gradients, and are able to provide a more careful analysis.
Whereas in previous algorithms, they use a constant step size throughout the training process. 
Consequently, when getting close to the minimizer, they must use \emph{an exponential growing batch size} to reduce the gradient noise and to prevent oscillations.

We also provide an error bound and sample complexity where, as in prior work, we assume the variance of stochastic gradients are \emph{uniformly bounded} by $\sigma^2$, i.e. $\alpha(P,\loss) = 0$ and $\beta(P,\loss) = \sigma^2$ (as before, we assume that the population loss $\calR(\cdot)$  is strongly-convex and smooth). We have the following corollary:

\begin{corollary}
\label{cor:sample_complexity}
Under the same assumptions and with the same hyper-parameters in Theorem \ref{thm:clipped_sgd} and letting $C_1 = 1$, with the probability at least $1-\delta$, we have the following error bound:
\begin{align}
\label{eqn:cor_objective_value}
 \calR(\theta^{N+1}) - \calR(\theta^*)
\leq 
O\left( \frac{\tau_u^3}{\tau_\ell^3} \cdot \frac{r_0 \log(1/\delta)^2}{ N^2} + \frac{\tau_u}{\tau_\ell^2} \cdot \frac{\sigma^2 \log(1/\delta)}{N}
\right),
\end{align}
where $r_0 = \calR(\theta^1) - \calR(\theta^*)$ is the initialization error. In other words, to achieve $\calR(\theta^{N+1}) - \calR(\theta^*) \leq \epsilon$ with probability at least $1-\delta$, we need 
$
O\left( \max \left( \sqrt{\frac{\tau_u^3}{\tau_\ell^3} \cdot \frac{r_0}{\epsilon}}
, \frac{\tau_u \sigma^2}{\tau_\ell^2  \epsilon} \right)\log \left(\frac{1}{\delta} \right) \right)
$ samples.
\end{corollary}
With our general results in place we now turn our attention to deriving some important consequences for mean estimation and linear regression.

%\section{CONSEQUENCES FOR HEAVY-TAILED PARAMETER ESTIMATION}
\section{Consequences for Heavy-tailed Parameter Estimation}
\label{sec:consequence_para_estimation}

In this section, we investigate the consequences of Theorem \ref{thm:clipped_sgd} for statistical estimation in the presence of heavy-tailed noise. We plug in the respective loss functions $\loss$, the terms $\alpha(P,\loss)$ and $\beta(P,\loss)$ capturing the underlying stochastic gradient distribution, in Theorem \ref{thm:clipped_sgd} to obtain high-probability bounds for the respective statistical estimators.

\subsection{Heavy-tailed Mean Estimation}
\label{sec:cor_mean_estimation}

We assume that the distribution $P$ has bounded covariance matrix $\Sigma$. Then \texttt{clipped-SGD} for mean estimation has the following guarantee.
\begin{corollary}
\label{cor:heavy_tailed_mean_estimation}
(Streaming Heavy-tailed Mean Estimation) Given samples $\z_1,\cdots, \z_N \in \real^p$ from a distribution $P$ and confidence level $\delta \in (0, 2e^{-1})$, the Algorithm \ref{algo:clipped_sgd} instantiated with a loss function $\loss(\theta,\z) = \frac{1}{2}\norm{\z-\theta}{2}^2$, an initial point $\theta^1 \in \real^p$, $\gamma = 144\log(2/\delta) + 1$, a learning rate $\eta_t = \frac{1}{t+\gamma}$, and a clipping level
\begin{equation*}
\lambda =  C_1\sqrt{\frac{ \gamma(\gamma-1)\norm{\theta^1 - \theta^*}{2}^2}{\log(2/\delta)^2}
+ \frac{(N+\gamma)\trace{\Sigma}}{\log(2/\delta)} },
\end{equation*}
where $C_1 \geq 1$ is a scaling constant can be chosen by users, returns $\theta^{N+1}$ such that with probability at least $1-\delta$, we have
\begin{equation}
\label{eqn:mean_estimation_bound}
\norm{\theta^{N+1} -\theta^*}{2}
\leq 100 \left(\frac{\gamma \norm{\theta^1 - \theta^*}{2}}{N+\gamma}
+ \sqrt{\frac{\trace{\Sigma}\log(2/\delta)}{N+\gamma}} \right).
\end{equation}
\end{corollary}

\paragraph{Remarks:}
a) The proposed mean estimator matches the error bound of the well-known geometric-median-of-means estimator~\citep{minsker2015geometric}, achieving $\norm{\theta^N- \theta^*}{2} \lesssim \sqrt{\frac{\trace{\Sigma}\log(1/\delta)}{N}}$. This guarantee is still sub-optimal compared to the optimal \emph{sub-Gaussian rate}~\citep{lugosi2019mean}. Existing polynomial time algorithms having optimal performance are for the batch setting and require either storing the entire dataset~\citep{hopkins2020mean,cherapanamjeri2019fast,diakonikolas2019robust,hopkins2020robust, cheng2020high, lei2020fast, dong2019quantum} or have $O(p\log(1/\delta))$ storage complexity~\citep{depersin2019robust}.
On the other hand, we argue that trading off some statistical accuracy for a large savings in memory and computation is favorable in practice.

Moreover, we claim that these algorithms are hard to be implemented in the streaming setting: \citet{hopkins2020mean, cherapanamjeri2019fast} use the semi-definite programming method, which is not yet practical. Algorithms \citet{diakonikolas2019robust, hopkins2020robust, cheng2020high, lei2020fast, dong2019quantum} rely on analyzing the spectrum of the covariance matrix. These approaches
require polylogarithmic passes over data to remove potential outliers in $d$ orthogonal directions, making it unsuitable in our setting since  computing the covariance matrix already requires taking one pass over data.

b) One may argue that median-based approaches, e.g.  coordinate-wise/geometric median-of-means or much simpler coordinate-wise/geometric medians, can be implemented in a streaming fashion while retaining the same rate as in the batch setting. 
Specifically, coordinate-wise/geometric median-of-means first divide the data into $b$ buckets of roughly equal size, compute the mean in each bucket, and then takes the coordinate-wise/geometric median of these bucketed means. 
We argue that they are not favorable for the following reasons.

First of all, simply using coordinate-wise/geometric medians of these $N$ samples is inconsistent. These medians are consistent estimators for the medians of underlying distributions. To use them in the mean estimation, it incurs a large bias when the underlying distribution is asymmetric. For instance, the distance between the mean and the median for a one-dimensional Pareto distribution is a constant. 

Second, it's not obvious how to turn these algorithms into the steaming setting. Current analyses for geometric median are asymptotic \citep{cardot2013efficient} or only applied after a large number of iterations \citep{cardot2017online} because estimating streaming geometric median incurs a large bias in the early iterations. Moreover, in practice, it might require one to wait for a bucket of samples to calculate bucketed means, which is not an 'any-time' algorithm as our algorithm.

Coordinate-wise median-of-means has similar issues. Also, even in the batch setting, the confidence parameter delta in their guarantee gets scaled by dimension, which can not be applied to very high-dimensional spaces \citep{minsker2015geometric}.

\subsection{Heavy-tailed Linear Regression}
\label{sec:cor_linear_regression}

We consider the linear regression model described in Eq.\eqref{eqn:def_linear_regression}. Assume that the covariates $x \in \real^p$ have bounded $4^{th}$ moments and a non-singular covariance  matrix $\Sigma$ with bounded operator norm, and the noise $w$ has bounded $2^{nd}$ moments. We denote the minimum and maximum eigenvalue of $\Sigma$ by $\tau_\ell$ and $\tau_u$. 
More formally,
we say a random variable $x \in \real^p$ has a bounded $4^{th}$ moment if there exists a constant $C_{4}$ such that for every unit vector $v \in \calS^{p-1}$, we have 
\begin{equation}
\label{eqn:asp_bounded_moment}
\Exp[\inprod{x-\Exp[x]}{v}^{4}]
\leq C_{4} \left(\Exp[\inprod{x-\Exp[x]}{v}] \right)^2.
\end{equation}
\begin{corollary}
\label{cor:heavy_tailed_linear_regression}
(Streaming Heavy-tailed Regression) Given samples $(x_1,y_1),\cdots,(x_N,y_N) \in \real^p \times \real$ and confidence level $\delta \in (0, 2e^{-1})$, the Algorithm \ref{algo:clipped_sgd} instantiated with loss function in Eq. \eqref{eqn:def_linear_regression}, initial point $\theta^1 \in \real^p$, 
$$
\gamma = 144\max \left\{\frac{\tau_u}{\tau_\ell}, \frac{192(C_4+1) p \tau_u^2 }{\tau_\ell^2} \right\}\log(2/\delta) + 1, $$
learning rate $\eta_t = \frac{1}{\tau_\ell(t+\gamma)}$, and clipping level
\begin{equation*}
\lambda =  C_1\sqrt{\frac{ \tau_\ell^2 \gamma(\gamma-1)\norm{\theta^1 - \theta^*}{2}^2}{\log(2/\delta)^2}
+ \frac{(N+\gamma)\sigma^2p \tau_u}{\log(2/\delta)} },
\end{equation*}
where $C_1$ is a scaling constant can be chosen by users and $C_4$ is the constant in Eq.\eqref{eqn:asp_bounded_moment}, returns $\theta^{N+1}$ such that with probability at least $1-\delta$, we have
\begin{equation}
\label{eqn:linear_regression_bound}
\norm{\theta^{N+1} -\theta^*}{2}
\leq 100 C_1 \left(\frac{\gamma \norm{\theta^1 - \theta^*}{2}}{N+\gamma}
+ \frac{\sigma}{\tau_\ell}\sqrt{\frac{p\tau_u\log(2/\delta)}{N+\gamma}} \right).
\end{equation}
\end{corollary}

\section{Proof Sketch of Theorem \ref{thm:clipped_sgd}}
\label{sec:sketch_of_proof}

In this section, we provide an overview of the arguments that constitute the proof of Theorem \ref{thm:clipped_sgd} and explain our theoretical contribution. The full details of the proof can be found in Section \ref{sec:proof_clipped_sgd} in Appendix. Our proof is improved upon previous analysis of clipped-SGD with constant learning rate and $O(n)$ batch size ~\citep{gorbunov2020stochastic} and high probability bounds for SGD with $O(1/t)$ step size in the sub-Gaussian setting~\citep{harvey2019simple}. Our analysis consists of three steps: (i) Expansion of the update rule. (ii) Selection of the clipping level. (iii) Concentration of bounded martingale sequences.

\paragraph{Notations:} Recall that we use step size $\eta_t = \frac{1}{\tau_\ell(t+\gamma)}$. We will write $\epsilon_t = \nabla \calR(\theta^{t}) - \clip(\nabla \loss(\theta^{t},\z_t), \lambda) $ is the noise indicating the difference between the stochastic gradient and the true gradient at step $t$. Let $\calF_t = \sigma(\z_1,\cdots,\z_t)$ be the $\sigma$-algebra generated by the first $t$ steps of \texttt{clipped-SGD}. We note that clipping introduce bias so that $\Exp_{\z_t}[\epsilon_t | \calF_{t-1}]$ is no longer zero, so we decompose the noise term $\epsilon_t = \epsilon_{t}^b + \epsilon_{t}^v$ into a bias term $\epsilon_{t}^b $ and a zero-mean variance term  $\epsilon_{t}^v $, i.e. 
$$\epsilon_t = \epsilon_{t}^b + \epsilon_{t}^v, 
\text{ where } 
\epsilon_{t}^b  = \Exp_{\z_t}[\epsilon_t| \calF_{t-1}]
\text{ and }
\epsilon_{t}^v  = \epsilon_t - \Exp_{\z_t}[\epsilon_t| \calF_{t-1}]
$$

\paragraph{(i) Expansion of the update rule:}  We start with the following lemma that is pretty standard in the analysis of SGD for strongly convex functions. It can be obtained by unrolling the update rules $\theta^{t+1} = \calP_\Theta \left(\theta^t - \eta_t \clip(\nabla \loss(\theta^t, \z_t),\lambda) \right)$ and using properties of $\tau_\ell$-strongly-convex and $\tau_u$-smooth functions.
\begin{lemma}
\label{lm:strongly_convex_expansion_main_paper}
Under the conditions in Theorem \ref{thm:clipped_sgd}, for any $1 \leq i \leq N$, we have
\begin{align*}
\small
\norm{\theta^{i+1} - \theta^*}{2}^2 
\leq  
\underbrace{\frac{\gamma(\gamma-1)\norm{\theta^1 - \theta^*}{2}^2 }{(i+\gamma)(i+\gamma-1)}}_{\text{the initialization error}} 
+ \underbrace{\frac{\sum_{t=1}^{i}  (t+\gamma-1) \inprod{\epsilon_t^b +\epsilon_t^v  }{\theta^t - \theta^*} }{\tau_\ell(i+\gamma)(i+\gamma-1)}}_{\text{the first noise term }} 
+ \underbrace{\frac{2\sum_{t=1}^{i} \left( \norm{ \epsilon_t^v }{2}^2 + \norm{ \epsilon_t^b }{2}^2 \right)} {\tau_\ell^2(i+\gamma)(i+\gamma-1)}}_{\text{the second noise term
}}. 
\end{align*}
\end{lemma}

\paragraph{(ii) Selection of the clipping level:} Now, to upper bound \textit{the noise terms}, 
we need to choose the clipping level $\lambda$ properly to balance the variance term $\epsilon_t^v$ and the bias term $\epsilon_t^b$. Specifically, we use the inequalities of \citet{gorbunov2020stochastic}, which provides us upper bounds for the magnitude and variance of these noise terms.
\begin{lemma}
(Lemma F.5, \citep{gorbunov2020stochastic} )
For any $t=1,2,..,N$, we have
\begin{equation}
\norm{\epsilon_t^v}{2}
\leq  2\lambda.
\end{equation}
Moreover, for all $t=1,2,..,N$, assume that the variance of stochastic gradients is bounded by $\sigma_t^2$, i.e. $\Exp_{\z_t}[\norm{\nabla \loss(\theta^t, \z_t) - \nabla \calR(\theta^t)}{2}^2 | \calF_{t-1}] \leq \sigma_t^2 $ and assume that the norm of the true gradient is less than $\lambda/2$, i.e.
$ \norm{\nabla \calR(\theta^t)}{2} \leq \lambda/2$. Then we have
\begin{equation}
\label{eqn:b_and_v_ub}
\norm{\epsilon_t^b}{2}
\leq \frac{4\sigma_t^2}{\lambda} 
\ \ \text{ and } \ \ 
\Exp_{\z_t}[\norm{\epsilon_t^v}{2}^2  | \calF_{t-1}] \leq 18 \sigma_t^2 \ \ 
\text{ for all } t=1,2,...,N.
\end{equation}
\label{lm:noise_term_control}
\end{lemma}
This lemma gives us the dependencies between the variance, bias and clipping level: a larger clipping level leads to a smaller bias but the magnitude of the variance term $\norm{\epsilon_t^v}{2}$ is larger, while the variance of the variance term $\Exp_{\z_t}[\norm{\epsilon_t^v}{2}^2  | \calF_{t-1}]$ remains constant. These three inequalities are also essential for us to use concentration inequalities for martingales. However, we highlight the \emph{necessity} for these inequalities to hold: \emph{the true gradient lies in the clipping region up to a constant}, i.e. $\norm{\nabla \calR(\theta^t)}{2} \leq \lambda / 2$. This condition is necessary since without this, we could not have upper bounds of the bias and variance terms. 
Therefore, the clipping level should be chosen in a very accurate way. Below we informally describe how do we choose it.

We note that $\norm{\theta^t - \theta^*}{2}$ should converge with $1/\sqrt{t}$ rate for strongly convex functions with $O(1/t)$ step size \citep{rakhlin2011making,harvey2019simple}. To make sure  \textit{the first noise term} upper bound by $O(1/i)$ , one should expect each summand  $t\inprod{\epsilon_t^b}{\theta^t-\theta^*} = O(1)$ for $1 \leq t \leq i$, which implies $\norm{\epsilon_t^b}{2} = O(1/\sqrt{t})$.
This motivates us to choose the clipping level to be proportional to $\sqrt{t}$ by Eq.\eqref{eqn:b_and_v_ub}. 
Also, from the detailed proof in Section \ref{sec:proof_clipped_sgd}, we will show that the delay parameter $\gamma$ makes sure $\norm{\nabla \calR(\theta^t)}{2} \leq \lambda / 2$ holds with high probabilities and the position dependent noise $\alpha(P,\loss)$ is controlled.

\paragraph{(iii) Concentration of bounded martingale sequences:} 
A significant technical step of our analysis is the use of the following Freedman's inequality.

\begin{lemma}
\label{lm:freedman_main_paper}
(Freedman's inequality \citep{freedman1975tail}) Let $d_1, d_2, \cdots,d_T$ be a martingale difference sequence with a uniform bound $b$ on the steps $d_i$. Let $V_s$ denote the sum of conditional variances, i.e. 
$ V_s = \sum_{i=1}^s \var(d_i | d_1,\cdots,d_{i-1}). $
Then, for every $a,v > 0$,
$$
\Pr \left(\sum_{i=1}^s d_i \geq a \ \text{ and } \ V_s \leq v \ \text{ for some }\ s \leq T \right)
\leq \exp \left(\frac{-a^2}{2(v+ba)} \right).
$$
\end{lemma}
Freedman's inequality says that if we know the boundness and variance of the martingale difference sequence, the summation of them has \emph{exponential concentration} around its expected value for \emph{all subsequence} $\sum_{i=1}^s d_i$.
 
Now we turn our attention to the variance term in \textit{the first noise term}, i.e. $\sum_{t=1}^{i}  (t+\gamma-1) \inprod{\epsilon_t^v  }{\theta^t - \theta^*}  $. It is the summation of a martingale difference sequence since $\Exp[\epsilon_t^v | \calF_{t-1}] = 0$.
Note that Lemma \ref{lm:noise_term_control} has given us upper bounds for boundness/variance for $\epsilon_t^v$.
However, the main technical difficulty is that each summand involves the error of past sequences, i.e. $\norm{\theta^t-\theta^*}{2}$. 

Our solution is the use of Freedman's inequality, which gives us a \emph{loose control} of all past error terms with high probabilities, i.e. $\norm{\theta^t - \theta^*}{2}^2 \leq O\left(N/t^2 \right)$ for $1 \leq t \leq N$. On the contrary, a recurrences technique used in the past works ~\citep{gorbunov2018accelerated,gorbunov2019optimal,gorbunov2020stochastic} uses an increasing clipping levels $\lambda_t$ and calls the Bernstein inequality ( , which only provides an upper bound for the \emph{entire sequence}, ) $N$ times in order to control $\norm{\theta^t - \theta^*}{2}$ for every $t$. As a result, 
it incurs an extra factor $\log(N)$ on their bound since it imposes a too strong control over past error terms.

Finally, we describe why \texttt{clipped-SGD} allows a $O(1)$ batch size at a high level. For previous works of stochastic optimization with strongly convex and smooth objective, they use a constant step size throughout their training process~\citep{gorbunov2018accelerated,prasad2018robust, davis2021low, nazin2019algorithms}. However, to ensure their approach make a constant progress at each iteration, they should use \emph{an exponential growing batch size} to reduce variance of gradients. Whereas in our approach, we explicitly control the variance by using \emph{a decayed learning rate and clipping the gradients}. Therefore, we are able to provide a careful analysis of the resulted bounded martingale sequences.

\section{Experiments}
%\section{EXPERIMENTS}
\label{sec:experiment}
To corroborate our theoretical findings, we conduct experiments on mean estimation and linear regression to study the performance of \texttt{clipped-SGD}. Another experiment on logistic regression is presented in the Section \ref{sec:logistic_regression} in the Appendix.

\paragraph{Methods.}
We compare \texttt{clipped-SGD} with vanilla SGD, which takes stochastic gradient to update without a clipping function. 
For linear regression, we also compare \texttt{clipped-SGD} with Lasso regression~\citep{nguyen2012robust}, Ridge regression and Huber regression ~\citep{sun2020adaptive,huber1973robust}.
All methods use the same step size  $\frac{1}{t+\gamma}$ at step $t$.

To simulate heavy-tailed samples, we draw from a scaled standardized Pareto distribution with tail-parameter $\beta$ for which the $k^{th}$ moment only exists for $k < \beta$. The smaller the $\beta$, the more heavy-tailed the distribution. Due to space constraints, we defer other results with different setups to the Appendix.

\paragraph{Choice of Hyper-parameter.}
Note that in Theorem \ref{thm:clipped_sgd}, the clipping level $\lambda$ depends on the initialization error, i.e. $\norm{\theta^1 - \theta^*}{2}$, which is not known in advance. Moreover, we found that the suggested value of $\gamma$ has a too large of a constant factor and substantially decreases the convergence rate especially for small $N$. However, standard hyper-parameter selection techniques
such as cross-validation and hold-out validation require storing the entire validation set, which are not suitable for streaming settings.

Consequently, we use a sequential validation method ~\citep{hyndman2018forecasting}, where we do "training" and "evaluation"
on the last $q$ percent of the data. Formally, given candidate solutions $\{ \hat{\theta}_1,\cdots, \hat{\theta}_m \}$ trained from samples $\z_1, \cdots, \z_N$, let $\theta^t_i$ be the estimated parameter for candidate $i$ at iteration $t$. Then we choose the candidate that minimizes the empirical mean of the risk of the last $q$ percents of samples, i.e. 
\begin{equation}
\small
\label{eqn:hyperparameter_selection}
j^* = \argmin_{1 \leq j \leq m} \frac{1}{qN} \sum_{t=(1-q)N+1}^{N} \loss(\widehat{\theta}_j^t, \z_t)
\end{equation}
Specifically, at the last $q$ percents of iterations, when a sample $\z_t$ comes, 
we first calculate the risk induced by $\theta^t_j$ and $\z_t$ and then use this sample to update the parameter $\theta^t_j$. 
Therefore, instead of storing the entire validation set, we only need $O(mp)$ space to store the candidate parameters $\{ \hat{\theta}_1,\cdots, \hat{\theta}_m \}$ and the validation losses of the candidates. 

In our experiment, we choose $q=0.2$ to tune the delay parameter $\gamma$, the clipping level $\lambda$, regularization factors for Lasso, Ridge and Huber regression, and the step size for streaming coordinate-wise/geometric median-of-means.

\paragraph{Metric.}  
For any estimator $\widehat{\theta}$, we use the $\ell_2$ loss $\ell(\widehat{\theta}) = \norm{\widehat{\theta} - \theta^*}{2}$ as our primary metric. To measure the tail performance of estimators (not just their expected loss), we also use $Q_{\delta}(\widehat{\theta}) = \inf(\alpha : \Pr(\ell(\widehat{\theta}) > \alpha) \leq \delta)$, which is the bound on the loss that we wish to hold with probability $1-\delta$. This could also be viewed as the $100(1 - \delta)$ percentile of the loss (e.g. if $\delta = 0.05$, then this would be the 95th percentile of the loss). 

\subsection{Synthetic Experiments: Mean estimation}
\label{sec:exp_mean_estimation}

\begin{figure*}[t]
\centering
      \subfigure
      %[\label{fig:mean_estimation_convergence}\small{Expected convergence curve.}]
      {\includegraphics[width=0.29\textwidth]{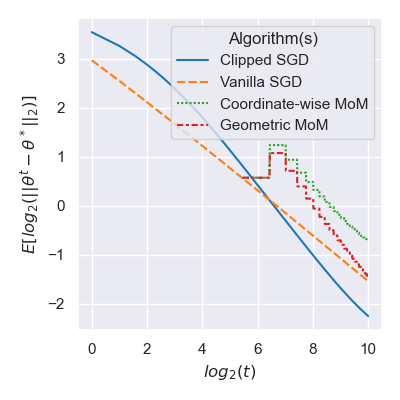}}\hfill
      \subfigure
      %[\label{fig:mean_estimation_quantile} \small{Last iteration quantile loss.}]
      {\includegraphics[width=0.29\textwidth]{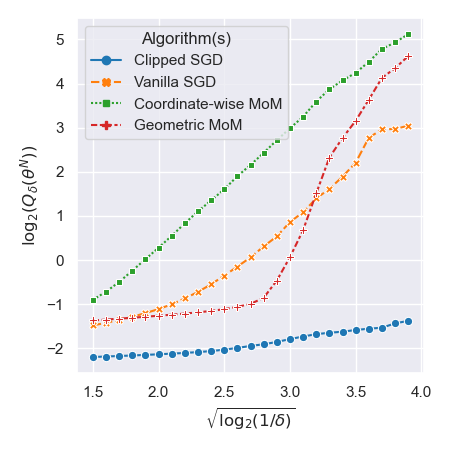}}\hfill
      \subfigure
      %[\label{fig:mean_estimation_lambda_vs_quantiles} \small{$\log_2(Q_\delta(\theta^N))$ v.s. $\lambda$.}]
      {\includegraphics[width=0.29\textwidth]{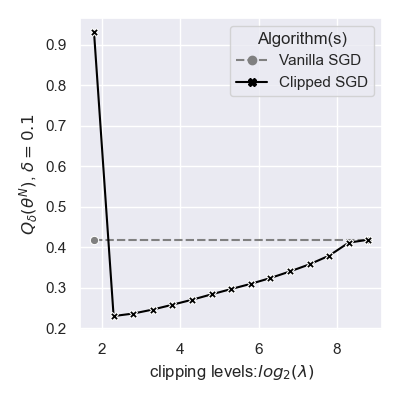}}%\hfill
      %\subfigure
      %[\label{fig:mean_estimation_initial_points} \small{Quantile loss v.s. different initial points}]
      %{\includegraphics[width=0.24\textwidth]{figs/mean_estimation/initialization_vs_quantile.png}}
\caption{Results for robust mean estimation for $N=1024$ and $p=256$. Smaller $Q_{\delta}(\widehat{\theta})$ is better.}
\label{fig:mean_estimation}
\end{figure*}

\paragraph{Setup.}
We obtain samples $\{\z_i\}_{i=1}^{N} \subseteq \real^p$ from a scaled standardized Pareto distribution with tail-parameter $\beta = 2.1$. We  initialize the mean parameter estimate as $\theta^1 = [1,\cdots, 1] \in \real^p$,
$\gamma = 0$, and fix the step size to $\eta_t = 1/t$ for \texttt{clipped-SGD} and Vanilla SGD. We note that in this setting, it can be seen that $\widehat{\theta}^t_{SGD} = \sum_{i=1}^t z_i/t$ is the empirical mean over $t$ samples by some simple algebra. 

We also compare our approach with streaming coordinate-wise/geometric median-of-means(MoM), where we use the number of buckets $b = \ceil{\log(1/\delta)}$ with $\delta = 0.05$  as in the batch setting~\citep{lugosi2019mean} and a step size of $\eta_t=\frac{c}{t_b}$ \citep{cardot2017online}, where $c$ is a constant selected by the validation method and $t_b$ is a counter for bucketed means. Specifically, we first wait for $\floor{\frac{n}{b}}$ points and calculate their running mean as a initial point. Then we calculate bucketed means for the remaining points in the same way and use them to update the coordinate-wise/geometric median with an averaged stochastic gradient algorithm. 
Each metric is reported over 50000 trials. See more implementation details in Appendix \ref{sec:exp_details}.

\paragraph{Results.}
Figure \ref{fig:mean_estimation} shows the performance of all algorithms. In the first panel, we can see \texttt{clipped-SGD} has \emph{expected} error that converges as $O(1/N)$ as our theory indicates. Also, the 99.9 percent quantile loss of Vanilla SGD is over 10 times worse than the expected error as in the second panel while the tail performance of \texttt{clipped-SGD} is similar to its expected performance. Moreover, streaming coordinate-wise/geometric median-of-means have a much worse expected performance and their tail performance is not well-controlled. 

In the third panel, we can see that \texttt{clipped-SGD} performs better across different $\lambda$. When the clipping level $\lambda$ is too small, it takes too small a step size so that the final error is high. While if we use a very large clipping level, the performance of \texttt{clipped-SGD} is similar to Vanilla SGD.

%In the fourth panel, we set initial point $\theta^1 = [x,x,\cdots,x]$, where $x \in [\frac{1}{16}, 1]$. For each initial point, we reselect hyperparameters for each algorithm via our sequential validation approach. From the figure,MoM-based approaches perform much worse as the initial point becomes further away from the true mean while Clipped/Vanilla SGD have similar performance across all initial points.

\subsection{Synthetic Experiments: Linear Regression}
\label{sec:exp_linear_regression}

\begin{figure*}[t]
\centering
      \subfigure[\label{fig:linear_regression_convergence} \small{Expected convergence curve. }]{\includegraphics[width=0.27\textwidth]{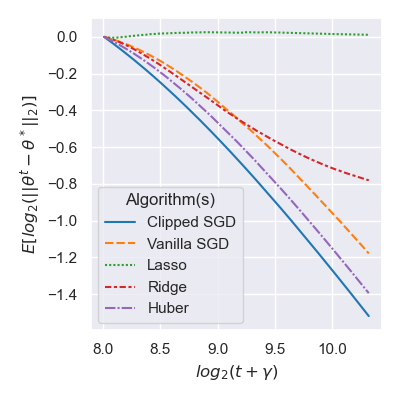} } 
      \hfill
      \subfigure[\label{fig:linear_regression_quantile} \small{Last iteration quantile loss. }]{\includegraphics[width=0.27\textwidth]{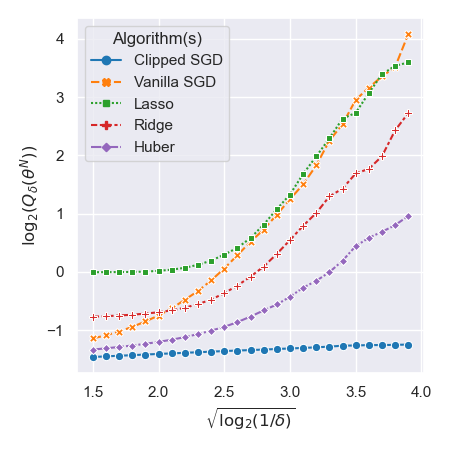}}
      \hfill
      \subfigure[\label{fig:linear_regression_lambda_vs_quantiles}\small{$\log_2(Q_\delta(\theta^N))$ v.s. $\lambda$.}]{\includegraphics[width=0.27\textwidth]{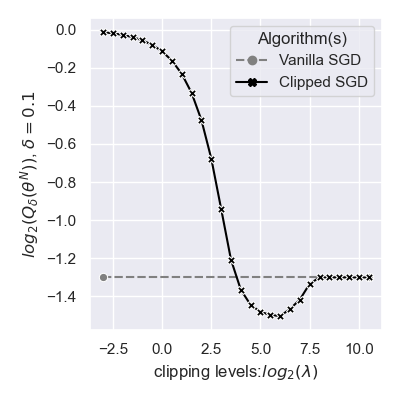}}

      \subfigure[\label{fig:linear_regression_gamma_vs_quantiles}\small{Expected convergence curve for different $\gamma$. }]{\includegraphics[width=0.27\textwidth]{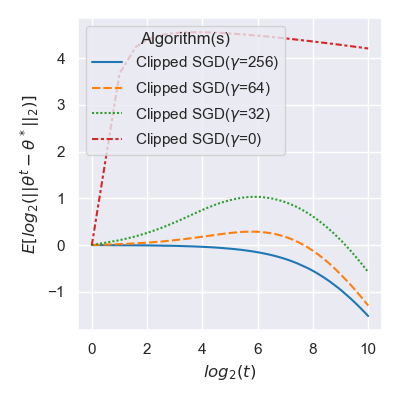} }\hfill
      \subfigure[\label{fig:linear_regression_sigma_vs_quantiles}\small{$\log_2(Q_\delta(\theta^N))$ v.s. $\sigma^2$ for $\delta = 0.01$. }]{\includegraphics[width=0.27\textwidth]{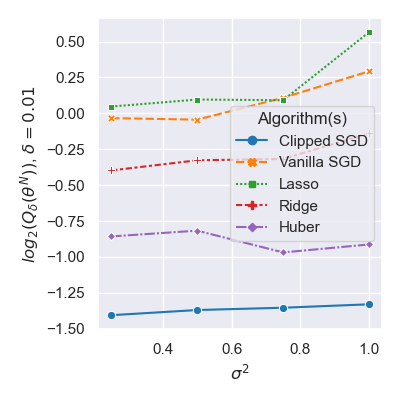}} \hfill
      \subfigure[\label{fig:linear_regression_p_vs_quantiles}\small{$\log_2(Q_\delta(\theta^N))$ v.s. $p$ for $\delta = 0.01$ and $N = 1024$.}]{\includegraphics[width=0.27\textwidth]{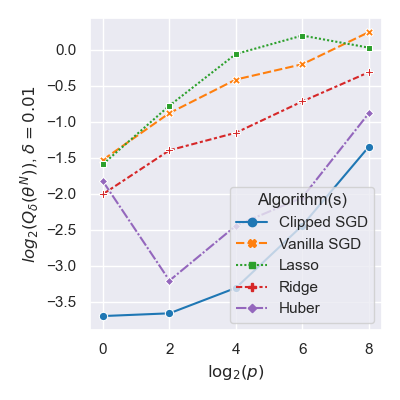}}
\caption{Results for robust linear regression for $N=1024, p = 256$, $\gamma=256$ and $\sigma^2 = 0.75$.}
\label{fig:linear_regression}
\end{figure*}

\paragraph{Setup}
We generate covariate  $x \in \real^p$ from an scaled standardized Pareto distribution with tail-parameter $\beta = 4.1$. The true regression parameter is set to be $\theta^* = [\frac{1}{\sqrt{p}},\cdots,\frac{1}{\sqrt{p}}] \in \real^p$ and the initial parameter is set to $\theta^1 = [0,0,\cdots,0]$. The response is generated by $y = \inprod{x}{\theta^*} + w$, where $w$ is sampled from scaled rescaled Pareto distribution with a zero mean, a variance $\sigma^2$ and a tail-parameter $\beta = 2.1$. We select $\tau_\ell = \lambda_{min}(\Sigma) = 1$. Each metric is reported over 50000 trials. 

\paragraph{Results}
We note that in this experiment, our hyperparameter selection technique yields $\gamma=256$. Figure \ref{fig:linear_regression_convergence}, \ref{fig:linear_regression_quantile} show that \texttt{clipped-SGD} performs the best among all baselines in terms of average error and quantile errors across different probability levels $\delta$. Also, $\sqrt{\log_2(1/\delta)}$ has linear relation to $Q_\delta(\widehat{\theta})$ as Corollary \ref{cor:heavy_tailed_linear_regression} indicates. In Figure \ref{fig:linear_regression_lambda_vs_quantiles}, we plot quantile loss against different clipping levels. It shows a similar trend to the mean estimation.

Next, in Figure \ref{fig:linear_regression_gamma_vs_quantiles}, we plot the averaged convergence curve for different delay parameters $\gamma$. This shows that it is necessary to use a large enough $\gamma$ to prevent it from diverging. This phenomenon can also be observed for different baseline models. Figure \ref{fig:linear_regression_sigma_vs_quantiles}, \ref{fig:linear_regression_p_vs_quantiles} shows that \texttt{clipped-SGD} performs the best across different noise level $\sigma^2$ and different dimension $p$.

%\section{CONCLUSION AND FUTURE DIRECTION}
\section{Conclusion and Future Direction}
\label{sec:conclusion}

In this paper, we provide a streaming algorithm for statistical estimation under heavy-tailed distribution. In particular, we close the gap in theory of clipped stochastic gradient descent with heavy-tailed noise. We show that $\texttt{clipped-SGD}$ can not only be used in parameter estimation tasks, such as mean estimation and linear regression, but also a more general stochastic optimization problem with heavy-tailed noise. There are several avenues for future work, including a better understanding of $\texttt{clipped-SGD}$ under different distributions, such as having higher bounded moments or symmetric distributions, where the clipping technique incurs less bias. Finally, it would also be of interest to extend our results to different robustness setting such as Huber's $\epsilon$-contamination model~\citep{huber1973robust}, where there are a constant portion of arbitrary outliers in observed samples.

\subsubsection*{Acknowledgements}
We acknowledge the support of ARL, NSF via OAC-1934584, and DARPA via HR00112020006.

\newpage

\begin{small}\bibliography{ref}
\bibliographystyle{plainnat}
\end{small}

%%%%%%%%%%%%%%%%%%%%%%%%%%%%%%%%%%%
%%%%%% SUPPLEMENT (OPTIONAL) %%%%%%
%%%%%%%%%%%%%%%%%%%%%%%%%%%%%%%%%%%

\clearpage
\appendix

\thispagestyle{empty}

% For one-column format, uncomment the following:
% For two-column format, uncomment the following:
%\twocolumn[ \makesupplementtitle ]

\appendix
\section{Organization}
The Appendices contain additional technical content and are organized as follows. In Appendix \ref{sec:detailed_related_work}, we provide additional details for related work, which contain a detailed comparison of previous work on stochastic optimization. In Appendix \ref{sec:exp_details}, we detail hyperparameters used for experiments in Section \ref{sec:experiment}. In Appendix \ref{sec:extra_exp}, we present supplementary experimental results for different setups for heavy-tailed mean estimation and linear regression. Additionally, we show a synthetic experiment on logistic regression.  Finally, in Appendix  \ref{sec:proof_clipped_sgd} and \ref{sec:proof_corollaries}, we give the proofs for Theorem \ref{thm:clipped_sgd} and corollaries respectively.

\section{Related work : Additional details}
\label{sec:detailed_related_work}

\paragraph{Heavy-tailed stochastic optimization.}
In this paragraph, we present a detailed comparison of existing results of stochastic optimization. In Table \ref{tab:sample_complexity_comparison}, we compare existing high probability bounds of stochastic optimization for strongly convex and smooth objectives. 

Since our work focus on large-scale 
setting (where we need to access data in a streaming fashion), we assume the number of samples $N$ is large so that the required $\epsilon$ is small. In such setting, $O(\frac{1}{\epsilon})$ is the dominating term and the error is driven by the stochastic noise term $\sigma^2$. If ignoring the difference in logarithmic factors and assuming $\tau_u/\tau_\ell$ is small, all methods for heavy-tailed noise in Table \ref{tab:sample_complexity_comparison} achieve $O(\frac{\sigma^2 }{\tau_\ell \epsilon})\log(\frac{1}{\delta})$ and are comparable to algorithms derived under the sub-Gaussian noise assumption. 

However, we can see that all of the existing methods require $O(\frac{1}{\epsilon})$ batch size except ours. Their batch sizes are not constants because they use a constant step size throughout their training process. Although they can achieve linear convergence rates for initialization error $r_0$, they should use \emph{an exponential growing batch size} to reduce variance induced by gradient noise. Additionally, in large scale setting where the noise term is the dominating term, the linear convergence of initial error is not important. 
On the contrary, we choose a $O(\frac{1}{T})$ step size, which is widely used in stochastic optimization for strongly convex objective~\citep{rakhlin2011making, harvey2019simple}. Our proposed \texttt{clipped-SGD} can therefore enjoy $O(\frac{1}{T})$ convergence rate while using a constant batch size.

Finally, our analysis improves the dependency on the confidence level term $\log(\frac{1}{\delta})$: it does not have extra logarithmic terms and does not depend on $\epsilon$.
Although our bound has a worse dependency on the condition number $\tau_u/\tau_\ell$, we argue that our bound has an extra $\tau_u/\tau_\ell$ term because our bound is derived under the square error, i.e. $\norm{\theta^t - \theta^*}{2}^2$ instead of the difference between objective values $\calR(\theta^t) - \calR(\theta^*)$. As a result, we believe the dependency on $\tau_u/\tau_\ell$ of our bounds can be improved by slightly revising Lemma \ref{lm:strongly_convex_expansion}.

\paragraph{Gradient clipping.}
Gradient clipping is a well-known optimization technique for both convex/non-convex optimization problems ~\citep{gulrajani2017improved,you2017scaling,pascanu2012understanding}. It has been shown to accelerate neural network training ~\citep{zhang2019gradient}, stablize the policy gradient algorithms ~\citep{garg2021proximal} and design different private optimization algorithms ~\citep{chen2020understanding,abadi2016deep}. Gradient clipping has also been shown to be robust to label noise ~\citep{li2020gradient}.

\newpage

\begin{table}[h]
    \centering
    \resizebox{\columnwidth}{!}{\begin{tabular}{|c|c|c|}
    \toprule
    Method & Sample complexity & Batch size \\
    \midrule
    \multicolumn{3}{|c|}{Sub-Gaussian noise} \\
    \toprule
    \begin{tabular}{@{}c@{}}
    \texttt{SIGMA2} \\         ~\citep{dvurechensky2016stochastic} 
    \end{tabular}
    & $O\left( \max \left( \frac{\tau_u}{\tau_\ell } \log(\frac{ r_0}{\epsilon}), \frac{\sigma^2}{\tau_\ell \epsilon}\log \left(\frac{1}{\delta}  \log(\frac{r_0}{\epsilon}) \right) \right) \right)$
    & $O \left( \frac{1}{\epsilon} \cdot \frac{\sigma^2}{\tau_\ell } \log \left(\frac{1}{\delta} \log(\frac{r_0}{\epsilon}) \right)\right)$\\
    \midrule
    \begin{tabular}{@{}c@{}}
    \texttt{MS-AC-SA} \\ ~\citep{ghadimi2013optimal} 
    \end{tabular}
     & $O\left( \max \left(   \sqrt{\frac{\tau_u}{\tau_\ell}} \log(\frac{ \tau_u r_0}{\tau_\ell \epsilon }), \frac{\sigma^2}{\tau_\ell \epsilon}\log \left(\frac{1}{\delta}  \log(\frac{\tau_u r_0}{\tau_\ell \epsilon}) \right) \right) \right)$
    & $O \left(\frac{1}{\epsilon} \cdot \frac{\sigma^2}{\tau_\ell \epsilon} \log \left(\frac{1}{\delta} \log(\frac{\tau_u r_0}{\tau_\ell \epsilon}) \right)\right)$\\
    \toprule

    \multicolumn{3}{|c|}{Heavy-tailed noise} \\
    \toprule

    \begin{tabular}{@{}c@{}}
    \texttt{restarted-} \\
    \texttt{RSMD}~\citep{nazin2019algorithms} 
    \end{tabular}
    & \begin{tabular}{@{}c@{}}
    $O\left( \max \left( \frac{\tau_u}{\tau_\ell} \log \left( \frac{\tau_\ell R^2}{\epsilon} \right) , \frac{\sigma^2}{\tau_\ell \epsilon}   \right) \cdot  C \right)$, \\
    where $ C =  \log \left( \frac{1}{\delta} \log(\frac{\tau_\ell R^2}{\epsilon})\right)  $
    \end{tabular}
    & $O \left( \frac{1}{\epsilon} \cdot \frac{\sigma^2}{\tau_\ell }  \log \left( \frac{ \log(\frac{\tau_\ell R^2}{\epsilon})}{\delta}\right) \right)$ \\
    \midrule
    \texttt{proxBoost} ~\citep{davis2021low} 
    & \begin{tabular}{@{}c@{}}
    $O\left( \max \left( \sqrt{\frac{\tau_u}{\tau_\ell }} \log \left( \frac{ \tau_u^2 r_0^2 \log(\frac{\tau_u}{\tau_\ell})}{\tau_\ell\epsilon} \right) , \frac{\sigma^2 \log \left(\frac{\tau_u}{\tau_\ell} \right)}{\tau_\ell \epsilon}   \right)  C' \right)$, \\
    where $ C' =  \log(\frac{\tau_u}{\tau_\ell})  \log \left( \frac{1}{\delta} \cdot \log(\frac{\tau_u}{\tau_\ell}) \right)  $
    \end{tabular}
    & $ O \left(  \frac{1}{\epsilon} \cdot \frac{\sigma^2 \log \left(\frac{\tau_u}{\tau_\ell} \right)}{\tau_\ell }  C' \right)$\\
    \midrule
    \texttt{RGD} ~\citep{prasad2018robust} 
    & \begin{tabular}{@{}c@{}}
    $O\left( \frac{\phi \tau_u \sigma^2}{\epsilon} \log\left(\frac{\phi}{\delta} \right)
    \right)$ 
    with $\phi = \log(\frac{\tau_u r_0}{\tau_\ell \epsilon}) / \log \left(\frac{\tau_u - \tau_\ell}{\tau_u + \tau_\ell} \right)  $
    \end{tabular}
    &  $O\left(  \frac{1}{\epsilon} \cdot \tau_u \sigma^2 \log\left(\frac{\phi}{\delta} \right)
    \right)$ \\
    \midrule
    \begin{tabular}{@{}c@{}}
    \texttt{clipped-SGD} \\
    (constant step size) \\
    ~\citep{gorbunov2020stochastic} 
    \end{tabular} 
    & \begin{tabular}{@{}c@{}}
    $O\left( \max \left( \frac{\tau_u}{\tau_\ell } , \frac{\tau_u\sigma^2}{\tau_\ell^2 \epsilon} \right) \cdot C \right)$, \\
    where $ C =  \log \left(\frac{r_0}{\epsilon} \right)  \log\left( \frac{\tau_u}{\tau_\ell \delta} \cdot \ \log(\frac{r_0}{\epsilon})  \right) $
    \end{tabular}
    & \begin{tabular}{@{}c@{}} 
    $O\Big(  \frac{1}{\epsilon} \cdot \frac{\sigma^2 \tau_u}{\tau_\ell^2  }\log(\frac{r_0}{\epsilon})  \cdot$ \\
    $\log( \frac{\tau_u}{\delta\tau_\ell} \log(\frac{r_0}{\epsilon})   \Big)$ 
    \end{tabular}\\
    \midrule
    \begin{tabular}{@{}c@{}}
    \texttt{R-clipped-} 
    \texttt{SGD} \\~\citep{gorbunov2020stochastic} 
    \end{tabular} 
    & \begin{tabular}{@{}c@{}}
    $O\left( \max \left( \frac{\tau_u}{\tau_\ell } \log(\frac{r_0}{\epsilon}), \frac{\sigma^2}{\tau_\ell \epsilon} \right) \cdot C \right)$, \\
    where $ C =\log \left(\frac{\tau_u}{\tau_\ell \delta} \right) +\log \left(\log(\frac{r_0}{\epsilon}) \right)  $
    \end{tabular}
    & $O\left(  \frac{1}{\epsilon} \cdot  \frac{\sigma^2}{\tau_\ell} \log(\frac{\tau_u}{\delta \tau_\ell}\log(\frac{r_0}{\epsilon})) \right)$  \\
    \midrule
    \begin{tabular}{@{}c@{}}
    \texttt{R-clipped-} 
    \texttt{SSTM} \\~\citep{gorbunov2020stochastic}
    \end{tabular}
    & \begin{tabular}{@{}c@{}}
    $O\left( \max \left( \sqrt{\frac{\tau_u}{\tau_\ell }} \log(\frac{r_0}{\epsilon}), \frac{\sigma^2}{\tau_\ell \epsilon} \right) \cdot C \right)$, \\
    where  $ C =\log \left(\frac{\tau_u}{\tau_\ell \delta} \right) + \log \left(\log(\frac{r_0}{\epsilon}) \right)  $
    \end{tabular}
    &  $O\left(   \frac{1}{\epsilon} \cdot  \frac{\sigma^2}{\tau_\ell } \log(\frac{\tau_u}{\delta \tau_\ell}\log(\frac{r_0}{\epsilon}) \right)$  \\
    \midrule
    \begin{tabular}{@{}c@{}}\texttt{clipped-SGD} \\
    ($O(1/T)$ step size) \\
    \lbrack This work \rbrack
    \end{tabular}
    & $
    O\left( \max \left( \sqrt{\frac{\tau_u^3}{\tau_\ell^3} \cdot \frac{r_0}{\epsilon}}
    , \frac{\tau_u \sigma^2}{\tau_\ell^2  \epsilon} \right)\log \left(\frac{1}{\delta} \right) \right)$
    & $O(1)$ \\
    \bottomrule
    \end{tabular}}
    \caption{Comparison of existing high probability upper bound for stochastic optimization for any $\tau_\ell$-strongly convex and $\tau_u$-smooth objective function $\calR(\cdot)$ with sub-Gaussian/heavy-tailed noise. The second column provides number of samples needed to achieved an $\epsilon$-approximated solution $\hat{\theta}$ such that $\calR(\hat{\theta}) -\calR(\theta^*) < \epsilon$ with probability at least $1- \delta$. In this table, we assume a gradient distribution has a uniformly bounded variance $\sigma^2$, i.e. $\alpha(P,\loss) = 0$ and $\beta(P,\loss) = \sigma^2$ in Eq.\eqref{eqn:asp_variance_sg}. 
    We use $r_0 = \calR(\theta^0) - \calR(\theta^*)$. For \texttt{RSMD}, $R$ is the diameter of the domain where the optimization problem is defined. The third column indicates the batch size, which is the number of samples used in a single step.}
    \label{tab:sample_complexity_comparison}
\end{table}

\newpage 

\section{Experimental Details}
\label{sec:exp_details}

\subsection{Experimental Details of Figure \ref{fig:intro_heavy_tail} }
\label{sec:details_for_intro_exp}
In Figure \ref{fig:intro_heavy_tail}, we consider the mean estimation task with a loss function $\calR(\theta) = \Exp_{\z} \left[ \norm{\z - \theta}{2}^2 \right]$, where $\z$ is either from a sub-Gaussian distribution or a Pareto distribution with tail parameter $2.1$. Both distributions are $10$-dimensional and have zero mean and an identity covariance matrix. We use $N=100$ samples and run $100,000$ trials to estimate each confidence level $\delta$. We choose the clipping level to be $\lambda = 1.5$.

\subsection{Experimental Details of Mean Estimation}
In this section, we describe the algorithms of streaming coordinate-wise/geometric median-of-means and present details of the synthetic experiment of mean estimation in Section \ref{sec:exp_mean_estimation} and Appendix \ref{sec:extra_mean_estimation}. 

\subsubsection{Streaming coordinate-wise/geometric median-of-means algorithms}

Given points $z_1,\cdots, z_n \in \real^d$, coordinate-wise/geometric medians of these $n$ points are defined as the minimizers of the following convex objective. 
\begin{equation}
\label{eqn:coordinate_wise_mom_obj}
\text{coordinate-wise median:} \ \
m_c \defeq \argmin_{u \in \real^d}
\sum_{i=1}^n \norm{z_i-u}{1}.
\end{equation}
\begin{equation}
\label{eqn:geometric_mom_obj}
\text{geometric median:} \ \ 
m_g \defeq \argmin_{u \in \real^d}
\sum_{i=1}^n \norm{z_i-u}{2}.
\end{equation}
These objectives are convex and therefore can be minimized via stochastic gradient descent \citep{cardot2013efficient,cardot2017online,feldman2007optimal}. In the experiment, we use the step sizes of $\frac{c}{t_b}$, where $c$ is a constant selected by using our sequential validation method and $t_b$ is the number of steps. In Algorithm \ref{algo:streaming_coordinate-wise_MoM} and \ref{algo:streaming_geometric_MoM} ,we describe the streaming coordinate-wise/geometric median-of-means algorithm.

\begin{algorithm}[H]
\caption{Streaming coordinate-wise median-of-means algorithm.}
\hspace*{\algorithmicindent} \textbf{Input:} step size $\eta_{t_b}$, number of buckets $b$, samples $\z_1, \cdots, \z_{N} \sim P$.
\begin{algorithmic}[1]
    \State Streaming calculate the bucketed mean $\theta^1 = \text{Average}(\z_{1}, \cdots, \z_{b})$.
    \For{$t_b=1,2,...,\floor{\frac{N}{b}}$}
    \State Streaming calculate the bucketed mean $\bar{z}_{t_b} = \text{Average}(\z_{bt_b+1}, \cdots, \z_{b(t_b + 1)})$.
    \State $\theta^{t_b+1} \leftarrow \theta^{t_b} - \eta_{t_b} \text{sign}(\theta^{t_b} - \bar{z}_{t_b}) $.
    \EndFor
\end{algorithmic}
\hspace*{\algorithmicindent} \textbf{Output:} $\theta^{\floor{\frac{N}{b}} + 1}$.
\label{algo:streaming_coordinate-wise_MoM}
\end{algorithm}
\begin{algorithm}[H]
\caption{Streaming geometric median-of-means  algorithm.}
\hspace*{\algorithmicindent} \textbf{Input:} initial point $\theta^1$, step size $\eta_{t_b}$, number of buckets $b$, samples $\z_1, \cdots, \z_{N} \sim P$.
\begin{algorithmic}[1]
    \State Streaming calculate the bucketed mean $\theta^1 = \text{Average}(\z_{1}, \cdots, \z_{b})$.
    \For{$t_b=1,2,...,\floor{\frac{N}{b}}$}
    \State Streaming calculate the bucketed mean $\bar{z}_{t_b} = \text{Average}(\z_{bt_b+1}, \cdots, \z_{b(t_b + 1)})$.
    \State $\theta^{t_b+1} \leftarrow \theta^{t_b} - \eta_{t_b} \frac{\theta^{t_b} - \bar{z}_{t_b}}{\norm{\theta^{t_b} - \bar{z}_{t_b}}{2}} $.
    \EndFor
\end{algorithmic}
\hspace*{\algorithmicindent} \textbf{Output:} $\theta^{\floor{\frac{N}{b}} + 1}$.
\label{algo:streaming_geometric_MoM}
\end{algorithm}

\subsubsection{Hyperparameter Selection}
For each setup $(N,p)$, we use the hyper-parameter selection technique described in Section \ref{sec:experiment} and choose the hyperparameters from the candidate sets in Table \ref{tab:candidate_set_mean_est}.
\\
\\
\\
\begin{table*}[h]
    \centering
    \begin{tabular}{|c|c|}
    \toprule
    Hyper-parameters & Candidate sets \\
    \midrule
    clipping level $\lambda$ & $\{ c\sqrt{Np} \  | \  c \in \{0.01, 0.06, \cdots,1.01 \} \} $ \\
    \midrule
    \begin{tabular}{@{}c@{}} Number of buckets for streaming\\ geometric/coordinate-wise median-of-means \end{tabular}
    & $\ceil{8 \log(1/\delta)}$ with $\delta = 0.05$ ~\citep{lugosi2019mean}\\
    \midrule
    \begin{tabular}{@{}c@{}} Step sizes for streaming\\ geometric/coordinate-wise median-of-means \end{tabular}
    & $\{\frac{c}{t_b} | c \in \{10^{-1},10^{-\frac{3}{4}},10^{-\frac{1}{2}},10^{-\frac{1}{4}},1,10^{\frac{1}{4}},10^{\frac{1}{2}},10^{\frac{3}{4}} \} \}$  \\
    \bottomrule
    \end{tabular}
    \caption{Candidate set for different hyper-parameters for mean estimation.}
    \label{tab:candidate_set_mean_est}
\end{table*}

\subsection{Experimental Details of Linear Regression}
In this section, we
present details of the synthetic experiment of linear regression in Section \ref{sec:exp_linear_regression} and \ref{sec:extra_linear_regression}. For each setup $(N,p)$, we use the hyper-parameter selection technique described in Section \ref{sec:experiment} and choose the hyperparameters from the candidate sets in Table \ref{tab:candidate_set_linear_regression}.

\begin{table*}[h]
    \centering
    \begin{tabular}{|c|c|}
    \toprule
    Hyper-parameters & Candidate sets \\
    \midrule
    delay parameter $\gamma$ & $\{0.1p, p, 10p \}$ \\
    \midrule
    clipping level $\lambda$ & $\{ c\sqrt{Np} \  | \  c \in \{0.01, 0.06, \cdots,1.01 \} \} $ \\
    \midrule
    regularization parameter for Lasso  & $\{ c\sqrt{Np} \  | \  c \in \{0.001, 0.006, \cdots,0.101 \} \} $ \\
     \midrule
    regularization parameter for Ridge  & $\{ c\sqrt{Np} \  | \  c \in \{0.001, 0.006, \cdots,0.101 \} \} $ \\
     \midrule
    regularization parameter for Huber & $\{ c\sqrt{Np} \  | \  c \in \{0.001, 0.006, \cdots,0.101 \} \} $ \\
    \bottomrule
    \end{tabular}
    \caption{Candidate sets for different hyper-parameters for linear regression.}
    \label{tab:candidate_set_linear_regression}
\end{table*}

\newpage

\section{Extra experiments}
\label{sec:extra_exp}

\subsection{Extra Experiments on Mean Estimation}
\label{sec:extra_mean_estimation}
Figure \ref{fig:extra_mean_estimation} shows extra experimental results for $p=20$ and $N=100,500,1000$. The experimental setting is the same as in Section \ref{sec:exp_mean_estimation}. We can see \texttt{clipped-SGD} consistently outperforms all other baselines in expected performance and tail performance.

\begin{figure}[!ht]
\centering
      \subfigure[\label{fig:mean_est_n100_p20_convergence} \small{Expected convergence curve for $N=100$ and $p=20$. }]{\includegraphics[width=0.31\textwidth]{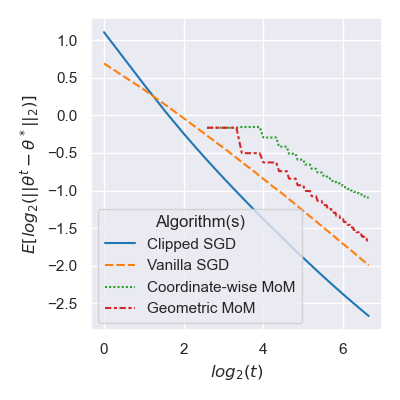} } \hfill
      \subfigure[\label{fig:mean_est_n100_p20_quantile} \small{Last iteration quantile loss for $N=100$ and $p=20$. }]{\includegraphics[width=0.31\textwidth]{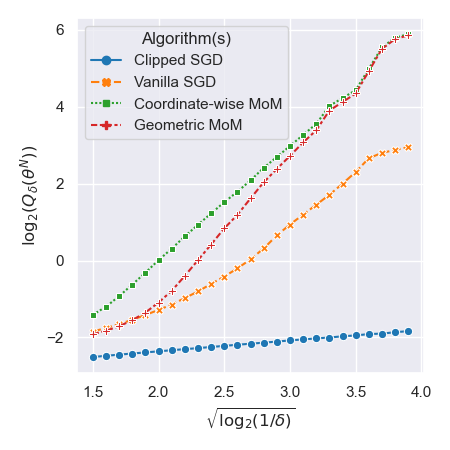}}
      \hfill
      \subfigure[\label{fig:mean_est_n500_p20_convergence} \small{Expected convergence curve for $N=500$ and $p=20$. }]{\includegraphics[width=0.31\textwidth]{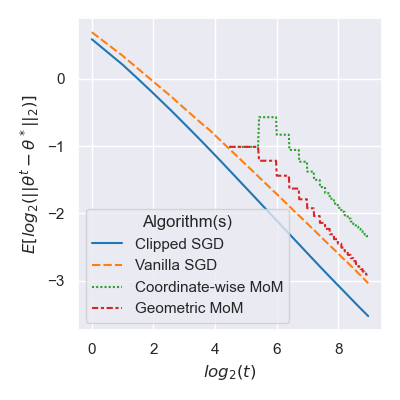} } 
     
      \subfigure[\label{fig:mean_est_n500_p20_quantile} \small{Last iteration quantile loss for $N=500$ and $p=20$. }]{\includegraphics[width=0.31\textwidth]{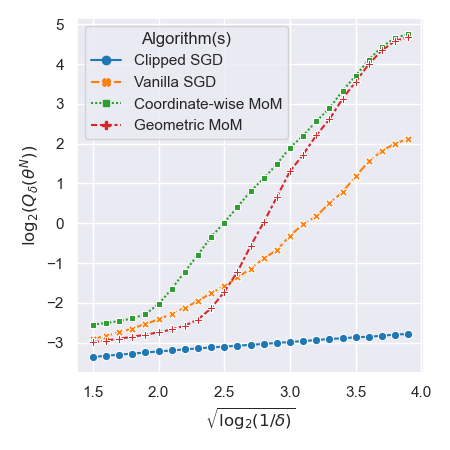}}
     hfill
       \subfigure[\label{fig:mean_est_n1000_p20_convergence} \small{Expected convergence curve for $N=1000$ and $p=20$. }]{\includegraphics[width=0.31\textwidth]{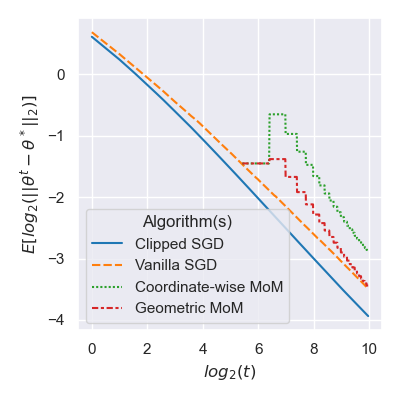} } \hfill
      \subfigure[\label{fig:mean_est_n1000_p20_quantile} \small{Last iteration quantile loss for $N=1000$ and $p=20$. }]{\includegraphics[width=0.31\textwidth]{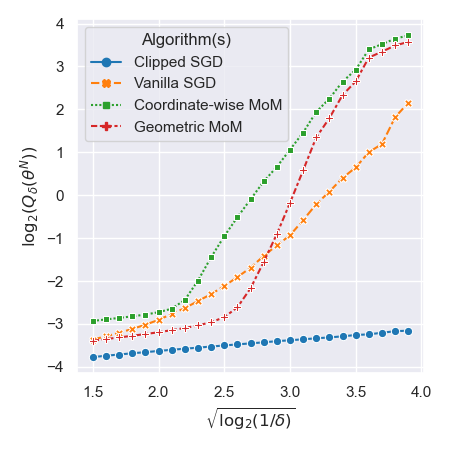}}

\caption{Results for robust mean estimation for different settings.}
\label{fig:extra_mean_estimation}
\end{figure}

\newpage

\subsection{Extra Experiments on Linear Regression}
\label{sec:extra_linear_regression}

Figure \ref{fig:extra_linear_regression} shows extra experimental results for $p=20$ and $N=100,500,1000$. The experimental setting is the same as in Section \ref{sec:exp_linear_regression}. We can see \texttt{clipped-SGD} consistently outperforms other baselines.

\begin{figure}[!ht]
\centering
      \subfigure[\label{fig:regression_n100_p20_convergence} \small{Expected convergence curve for $N=100$ and $p=20$. }]{\includegraphics[width=0.31\textwidth]{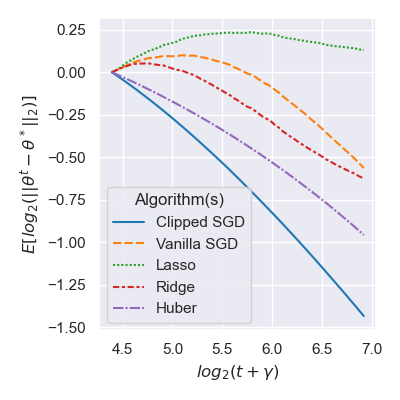} } \hfill
      \subfigure[\label{fig:regression_n100_p20_quantile} \small{Last iteration quantile loss for $N=100$ and $p=20$. }]{\includegraphics[width=0.31\textwidth]{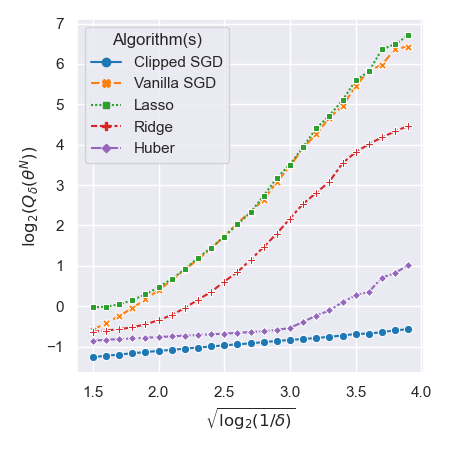}}
      \hfill
      \subfigure[\label{fig:regression_n500_p20_convergence} \small{Expected convergence curve for $N=500$ and $p=20$. }]{\includegraphics[width=0.31\textwidth]{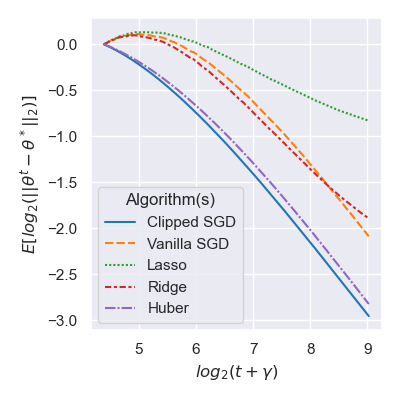} } 
     
      \subfigure[\label{fig:regression_n500_p20_quantile} \small{Last iteration quantile loss for $N=500$ and $p=20$. }]{\includegraphics[width=0.31\textwidth]{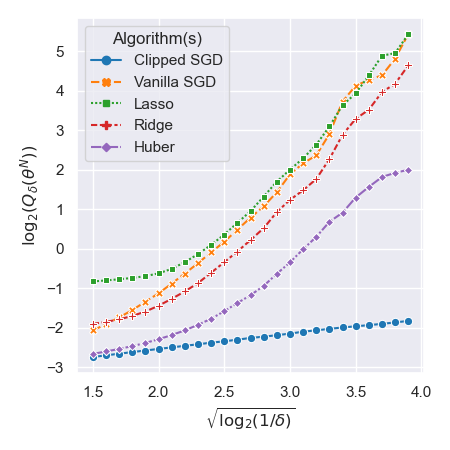}}
      \vspace{-0.12cm}\hfill
       \subfigure[\label{fig:regression_n1000_p20_convergence} \small{Expected convergence curve for $N=1000$ and $p=20$. }]{\includegraphics[width=0.31\textwidth]{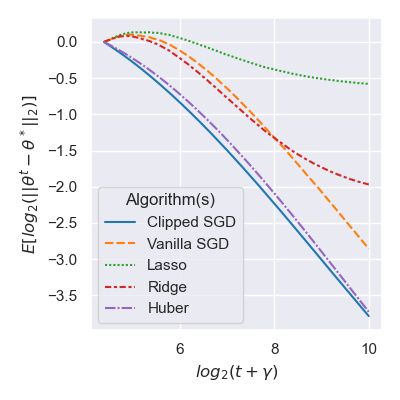} } \vspace{-0.12cm}\hfill
      \subfigure[\label{fig:regression_n1000_p20_quantile} \small{Last iteration quantile loss for $N=1000$ and $p=20$. }]{\includegraphics[width=0.31\textwidth]{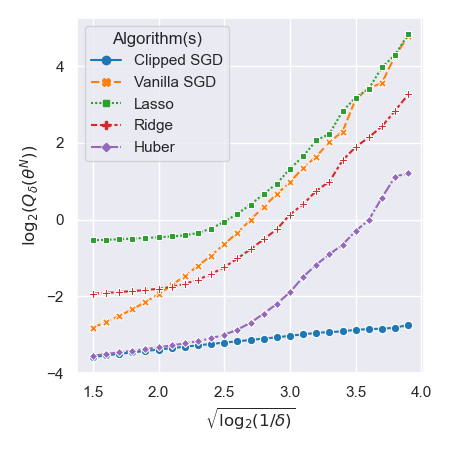}}
      \vspace{-0.12cm}\hfill
\caption{Results for robust linear regression for $p=20$, $\gamma = 20$, $\sigma^2 = 0.75$ and $N=100,500,1000$.}
\label{fig:extra_linear_regression}
\end{figure}

Moreover, we compare our algorithm when the true parameter $\theta^*$ is sparse, which is favorable for regularized linear regression (esp. Lasso). Specifically, we set the true parameter 
$\theta^*= [\frac{1}{\sqrt{r}},\cdots,\frac{1}{\sqrt{r}},0,\cdots, 0] \in \real^p$, where $r$ is a sparsity parameter indicating the last $p-r$ dimensions of $\theta^*$  are zero. We generate covariate  $x \in \real^p$ from an scaled standardized Pareto distribution with tail-parameter $\beta = 4.1$. The initial parameter is set to $\theta^1 = [\frac{-1}{\sqrt{p}},\frac{-1}{\sqrt{p}},\cdots,\frac{-1}{\sqrt{p}}]$. The response is generated by $y = \inprod{x}{\theta^*} + w$, where $w$ is sampled from scaled rescaled Pareto distribution with mean $0$, variance $\sigma^2$ and tail-parameter $\beta = 2.1$. 

Figure \ref{fig:extra_linear_regression_sparse} shows the results of sparse linear regression when $p=20$,$r=3$, $N=500,1000$. In sparse setting, Lasso performs better than in the dense setting but is still worse than our \texttt{clipped-SGD}.

\begin{figure}[!ht]
\centering
      \subfigure[\label{fig:sparse_regression_n500_p20_convergence} \small{Expected convergence curve for $N=500$ and $p=20$. }]{\includegraphics[width=0.23\textwidth]{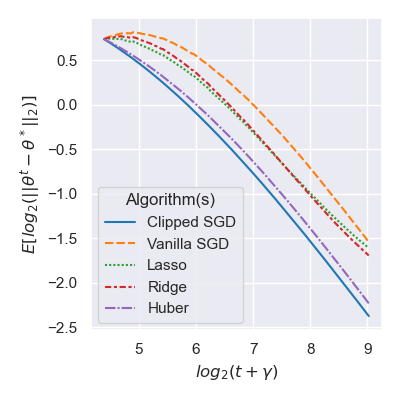} } \hfill
      \subfigure[\label{fig:sparse_regression_n500_p20_quantile} \small{Last iteration quantile loss for $N=500$ and $p=20$. }]{\includegraphics[width=0.23\textwidth]{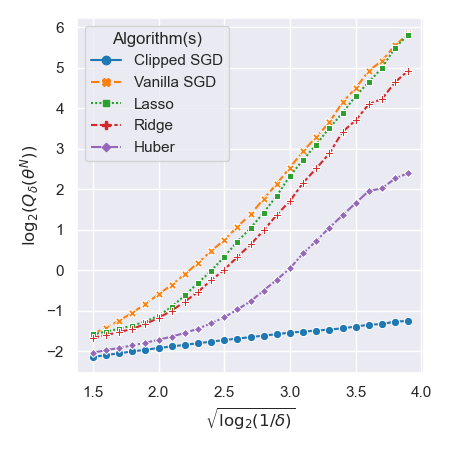}}\hfill
      \subfigure[\label{fig:sparse_regression_n1000_p20_convergence} \small{Expected convergence curve for $N=1000$ and $p=20$. }]{\includegraphics[width=0.23\textwidth]{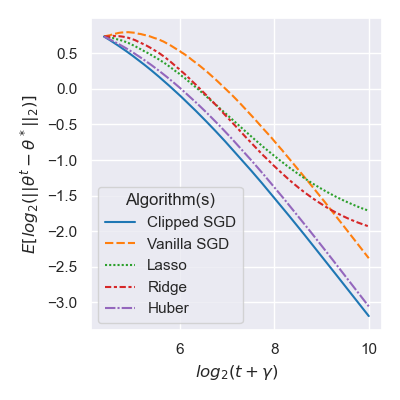} } \hfill
      \subfigure[\label{fig:sparse_regression_n1000_p20_quantile} \small{Last iteration quantile loss for $N=500$ and $p=20$. }]{\includegraphics[width=0.23\textwidth]{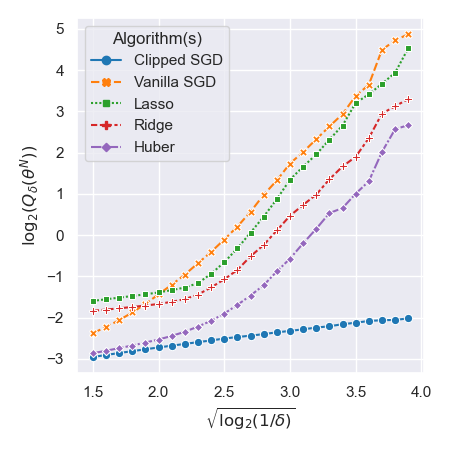}}\hfill
\caption{Results for robust sparse linear regression for $p=20$, $r=3$, $\sigma^2 = 0.75$, $\gamma=20$ and $N=500,1000$.}
\label{fig:extra_linear_regression_sparse}
\end{figure}

\newpage

\subsection{Synthetic Experiments: Logistic regression}
\label{sec:logistic_regression}
In this section, we present an extra experimental results on logistic regression.

\paragraph{Logistic regression model.} 
In this model, we observe covariate-response pairs $(x_i,y_i) \in \real^{p} \times \{0, 1\}$ for $1 \leq i \leq N$, where each pair $(x,y)$ is sampled from the true distribution $P$. The conditional distribution of the response $y$ given the covariate $x$ is 
\begin{equation}
\label{eqn:logistic_regression_cond_dist}
\Pr(y|x) = 
\begin{cases}
\frac{1}{1+ \exp(-\inprod{x}{\theta^*})} \text{, if } y=1. \\
\frac{1}{1+ \exp(\inprod{x}{\theta^*})} \text{, if } y=0. \\
\end{cases}
\end{equation}
We focus on the random design setting where the covariates $x \in \real^p$ have mean 0 and covariance matrix $\Sigma$. The loss function we used is negative log-likelihood function:
\begin{equation}
\label{eqn:neg_log_likelihood}
\loss(\theta, (x,y))
= -y \inprod{x}{\theta} + 
\log(1 + \exp(\inprod{x}{\theta})).
\end{equation}

The true parameter is the minimizer of the resulting population risk. The gradient of the loss function is given by 
\begin{equation}
\label{eqn:gradient_logistic_regression}
\nabla \loss(\theta, (x,y)) = \left(y - \frac{1}{1 + \exp(-\inprod{x}{\theta} )}\right)x.
\end{equation}
The hessian matrix of the population risk is
\begin{equation}
\label{eqn:hessian_logistic_regression}
\nabla^2 \calR(\theta) = \Exp \left[ \frac{\exp(\inprod{x}{\theta})}{(1 + \exp(\inprod{x}{\theta}))^2} xx^{\top}\right].
\end{equation}
We note that $\lambda_{min}(\nabla^2 \calR(\theta))$ approaches to $0$ as $\theta$ diverges and the objective function is no longer strongly convex. Therefore, in this case, we restrict the domain of $\theta$ to be a bounded convex set $\Theta$.

\paragraph{Setup.}
We generate covariate  $x \in \real^p$ from an scaled standardized Pareto distribution with tail-parameter $\beta = 4.1$. The true regression parameter is set to be $\theta^* = [\frac{1}{\sqrt{p}},\cdots,\frac{1}{\sqrt{p}}] \in \real^p$ and the initial parameter is set to $\theta^1 = 0.75 \theta^*$. The response is generated by Eq.\eqref{eqn:logistic_regression_cond_dist}. We select $\tau_\ell = \lambda_{min}(\Sigma) = 1$. To ensure $\lambda_{min}(\nabla^2 \calR(\theta))$ is lower bounded, we restrict the domain $\Theta$ to a unit ball centered at $\theta^*$, i.e. $\Theta = \{ v \ | \  \norm{v-\theta^*}{2} \leq 1 \}$. We set $\tau_\ell = 0.1$. Each metric are reported over 5000 trials. 

\paragraph{Results.}
Figure \ref{fig:extra_logistic_regression} shows the results of heavy-tailed logistic regression. In Figure \ref{fig:logistic_regression_gamma_vs_convergence}, we plot expected convergence curves for different $\gamma$ for SGD algorithm. We can see that $\gamma = 2000$ yields the best performance. Therefore, we fix $\gamma=2000$ and compare SGD algorithm with \texttt{clipped-SGD}. Figure \ref{fig:logistic_regression_convergence_vs_clipping_level}, shows expected convergence curves for different clipping levels $\lambda$. We can see that the red curve ($\lambda=0.5$) clearly outperforms Vaniila SGD. The tail performance of $\lambda=0.5$ is also the best as in Figure \ref{fig:logistic_regression_quantile}.

We note that the tail performance of Vanilla SGD is well-controlled for logistic regression, as can be seen in Figure \ref{fig:logistic_regression_quantile}.
The reason may be that the distribution of stochastic gradient is not as heavy as the distribution of covariate $x$. If we see the formula of stochastic gradients in Eq.\eqref{eqn:gradient_logistic_regression}, when $\norm{x}{2}$ is large, the response $y$ has high probability to be exponentially close to $1/(1+\exp(-\inprod{x}{\theta}))$. Therefore, the term inside the bracket is exponentially small with high probability. The stochastic gradient may not be as heavy-tailed as the covariate $x$. However, our results show that using clipped gradient is helpful in logistic regression.
 
\begin{figure}[!ht]
\centering
      \subfigure[\label{fig:logistic_regression_gamma_vs_convergence} \small{Expected convergence curves for different $\gamma$. }]{\includegraphics[width=0.31\textwidth]{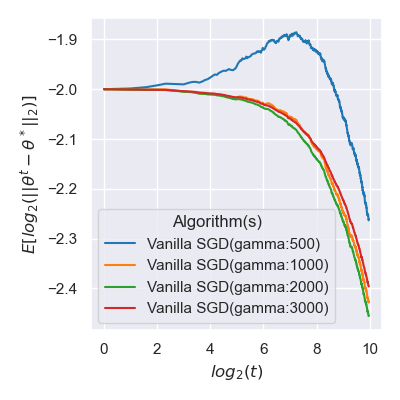} } \vspace{-0.12cm}\hfill
      \subfigure[\label{fig:logistic_regression_convergence_vs_clipping_level} \small{Expected convergence curves for different $\lambda$ ($\gamma = 2000$). }]{\includegraphics[width=0.31\textwidth]{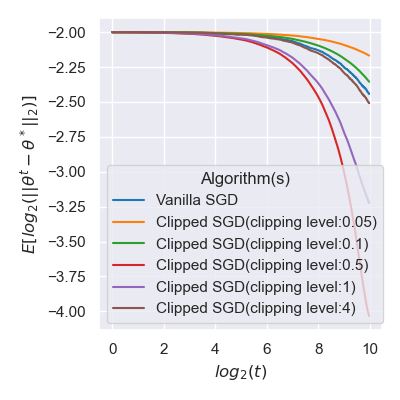}}
      \vspace{-0.12cm}\hfill
      \subfigure[\label{fig:logistic_regression_quantile} \small{Quantile loss for different clipping level $\lambda$. }]{\includegraphics[width=0.31\textwidth]{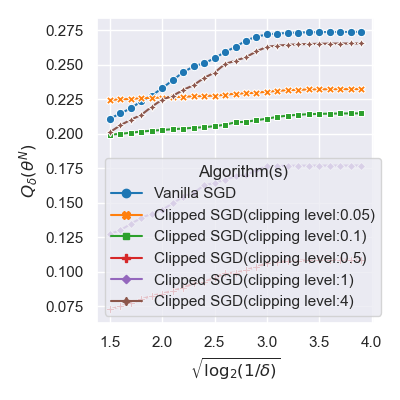} } \vspace{-0.12cm}
\caption{Results for heavy-tailed logistic regression for $p=10$ and $N=1000$.}
\label{fig:extra_logistic_regression}
\end{figure}

\newpage

\newpage

\section{Proof of Theorem \ref{thm:clipped_sgd}}
\label{sec:proof_clipped_sgd}

\begin{proof}

First of all, we let the clipped gradient at step $t$ be $g_t = \clip(\nabla \loss(\theta^{t},\z_t), \lambda)$ and 
let $\epsilon_t = \nabla \calR(\theta^{t}) - g_t $ be the difference between the stochastic gradient and the true gradient at step $t$. Also, we let $\calF_t= \sigma(\z_1,\cdots,\z_t)$ be the $\sigma$-algebra generated by the first $t$ steps of \texttt{clipped-SGD}. Our first step is unrolling the update rule:  $\theta^{t+1} = \calP_\Theta \left(\theta^t - \eta_t \clip(\nabla \loss(\theta^t, \z_t),\lambda) \right)$.

\begin{lemma}
\label{lm:strongly_convex_smooth_bound}
[Lemma 3.11, \citep{bubeck2014convex}] Let $f$ be $M$-smooth and $m$-strongly convex in $\real^p$, then for all $x,y \in \real^p$, we have
$$
\inprod{\nabla f(x) - \nabla f(y)}{x -y} 
\geq \frac{mM}{m+M} \norm{x-y}{2}^2 + \frac{1}{m+M} \norm{\nabla f(y) - \nabla f(x)}{2}^2.
$$
\end{lemma}

\begin{lemma}
\label{lm:strongly_convex_expansion}
Under the conditions in theorem \ref{thm:clipped_sgd}, for any $ 1 \leq t \geq N$, we have
\begin{equation}
\norm{\theta^{t+1} - \theta^*}{2}^2 
\leq  \frac{\gamma(\gamma-1)}{(t+\gamma)(t+\gamma-1)} \norm{\theta^1 - \theta^*}{2}^2 
+ \frac{\sum_{i=1}^{t}  (i+\gamma-1) \inprod{\epsilon_i}{\theta^i - \theta^*} }{\tau_\ell(t+\gamma)(t+\gamma-1)} 
+  \frac{\sum_{i=1}^{t} \norm{ \epsilon_i }{2}^2}{2\tau_\ell^2(t+\gamma)(t+\gamma-1)}. 
\end{equation}
\end{lemma}
\begin{proof}
By the strong convexity of $\calR(\theta)$ and the fact that $\theta^*$ minimizes $\calR(\theta)$ in $\Theta$, we have
\begin{equation}
\inprod{\nabla \calR(\theta^{t})}{\theta^t - \theta^*} \geq \calR(\theta^t) - \calR(\theta^*) + \frac{\tau_\ell}{2} \norm{\theta^t - \theta^*}{2}^2,
\end{equation}
and
\begin{equation}
\calR(\theta^t) - \calR(\theta^*)  \geq  \frac{\tau_\ell}{2} \norm{\theta^t - \theta^*}{2}^2.
\end{equation}
Putting these two inequality together, we have
\begin{equation}
\label{eqn:strong_convex}
\inprod{\nabla \calR(\theta^{t})}{\theta^t - \theta^*} \geq \tau_\ell \norm{\theta^t - \theta^*}{2}^2.
\end{equation}

Also, since $\Theta$ is a convex set, we have $\norm{\calP_{\Theta}(\theta) - \theta^*}{2} \leq \norm{\theta - \theta^*}{2}$ since $\theta^* \in \Theta$. By rewinding the update rule of clipped-SGD algorithm, we have the following:

\begin{align}
& \norm{\theta^{t+1} - \theta^*}{2}^2 \\
& = \norm{\calP_{\Theta} (\theta^t - \eta_t g_t) - \theta^*}{2}^2  \nonumber \\
& \leq \norm{\theta^t - \eta_t g_t - \theta^*}{2}^2  \nonumber \\
& = \norm{\theta^t - \theta^*}{2}^2 - 2 \eta_t \inprod{\theta^t - \theta^*}{g_t} + \eta_t^2 \norm{g_t}{2}^2 \nonumber \\
& = \norm{\theta^t - \theta^*}{2}^2 - 2\eta_t  \inprod{\nabla \calR(\theta^t)}{\theta^t - \theta^*} 
+ 2\eta_t  \inprod{\epsilon_t}{\theta^t - \theta^*} 
+ \eta_t^2\norm{\nabla \calR(\theta^t) + \epsilon_t }{2}^2
\nonumber \\
& \leq \norm{\theta^t - \theta^*}{2}^2 - 2\eta_t  \inprod{\nabla \calR(\theta^t)}{\theta^t - \theta^*} 
+ 2\eta_t  \inprod{\epsilon_t}{\theta^t - \theta^*} 
+ 2\eta_t^2\norm{\nabla \calR(\theta^t) }{2}^2
+ 2\eta_t^2\norm{ \epsilon_t }{2}^2,
\label{eqn:convex_expansion_mid}
\end{align}
where the last inequality is by Cauchy-Schwarz inequality. Then we apply Lemma \ref{lm:strongly_convex_smooth_bound} by initiating with $f = \calR(\cdot)$, $x = \theta^t$ and $y= \theta^*$. By the regularity assumption in Eq. \eqref{eqn:convexity_parameter}, we have 
$$
(\tau_\ell+\tau_u)\inprod{\nabla \calR(\theta^t)}{\theta^t - \theta^*} - \tau_\ell \tau_u\norm{\theta^t - \theta^*}{2}^2
\geq \norm{\nabla \calR(\theta^t)}{2}^2.
$$

Combining the above inequality and Eq. \eqref{eqn:convex_expansion_mid}, we have
\begin{align}
& \norm{\theta^{t+1} - \theta^*}{2}^2 \nonumber \\
& \leq (1 - 2 \eta_t^2\tau_\ell \tau_u )\norm{\theta^t - \theta^*}{2}^2 - (2\eta_t - 2\eta_t^2(\tau_\ell+\tau_u)) \inprod{\nabla \calR(\theta^t)}{\theta^t - \theta^*} + 2\eta_t  \inprod{\epsilon_t}{\theta^t - \theta^*} 
+ 2\eta_t^2\norm{ \epsilon_t }{2}^2
\nonumber \\
& \overset{(i)}{\leq} (1 - 2 \eta_t^2\tau_\ell \tau_u )\norm{\theta^t - \theta^*}{2}^2 - (2\eta_t - 2\eta_t^2(\tau_\ell+\tau_u))\tau_\ell\norm{\theta^t - \theta^*}{2}^2 + 2\eta_t  \inprod{\epsilon_t}{\theta^t - \theta^*} 
+ 2\eta_t^2\norm{ \epsilon_t }{2}^2
\nonumber \\
& = (1 - 2 \eta_t\tau_\ell - 2 \eta_t^2\tau_\ell^2 )\norm{\theta^t - \theta^*}{2}^2
 + 2\eta_t  \inprod{\epsilon_t}{\theta^t - \theta^*} 
+ 2\eta_t^2\norm{ \epsilon_t }{2}^2
\nonumber \\
& \leq (1 - 2 \eta_t\tau_\ell )\norm{\theta^t - \theta^*}{2}^2
 + 2\eta_t  \inprod{\epsilon_t}{\theta^t - \theta^*} 
+ 2\eta_t^2\norm{ \epsilon_t }{2}^2
\nonumber \\
& \leq \left(1-\frac{2}{t+\gamma} \right)\norm{\theta^t - \theta^*}{2}^2
 + \frac{2}{\tau_\ell(t+\gamma)}  \inprod{\epsilon_t}{\theta^t - \theta^*} 
 + \frac{2\norm{ \epsilon_t }{2}^2}{\tau_\ell^2(t+\gamma)^2},
\end{align}
where $(i)$ is due to $\ \forall \ t \geq 0, 2\eta_t \geq 2\eta_t^2(\tau_\ell + \tau_u)  \Leftrightarrow \gamma \geq \tau_u / \tau_\ell$ and Eq. \eqref{eqn:strong_convex}. The last inequality follows from the definition of $\eta_t$. Then we unwind this formula till $t=1$ and we get that for any $t > 1$,
\begin{align}
\label{eqn:unwind_function}
& \norm{\theta^{t+1} - \theta^*}{2}^2 \nonumber \\
& \leq \left(\prod_{j=1}^t \frac{j + \gamma -2}{j+\gamma} \right) \norm{\theta^1 - \theta^*}{2}^2 
+ \sum_{i=1}^{t} \frac{\inprod{\epsilon_i}{\theta^i - \theta^*} }{\tau_\ell(i+\gamma)} \left(\prod_{j=i+1}^t \frac{j + \gamma -2}{j+\gamma} \right) 
+  \sum_{i=1}^{t} \frac{\norm{ \epsilon_i }{2}^2}{2\tau_\ell^2(i+\gamma)^2} \left(\prod_{j=i+1}^t \frac{j + \gamma -2}{j+\gamma} \right) \nonumber \\
& \leq  \frac{\gamma(\gamma-1)}{(t+\gamma)(t+\gamma-1)} \norm{\theta^1 - \theta^*}{2}^2 
+ \frac{\sum_{i=1}^{t}  (i+\gamma-1) \inprod{\epsilon_i}{\theta^i - \theta^*} }{\tau_\ell(t+\gamma)(t+\gamma-1)} 
+  \frac{\sum_{i=1}^{t} \norm{ \epsilon_i }{2}^2}{2\tau_\ell^2(t+\gamma)(t+\gamma-1)} ,
\end{align}
where the last equation is due to 
$$
\prod_{j=i+1}^t \frac{j + \gamma -2}{j+\gamma}
= \frac{(i+\gamma)(i+\gamma-1)}{(t+\gamma)(t+\gamma-1)}.
$$
\end{proof}

Now come back to the proof of Theorem \ref{thm:clipped_sgd}. We note that clipping introduces bias, which influences the convergence of this method. Hence, we decompose the noise term $\epsilon_i = \clip(\nabla \loss(\theta^i,\z_i), \lambda) - \nabla \calR(\theta^i)$ into a bias term $\epsilon_{i}^b $ and a variance term  $\epsilon_{i}^v $, i.e. 
\begin{equation}
\epsilon_i = \epsilon_{i}^b + \epsilon_{i}^v, 
\ \text{ where } \  \epsilon_{i}^b  = \Exp_{\z_i}[\epsilon_i| \calF_{i-1}] 
\ \text{ and } \ 
\epsilon_{i}^v  = \epsilon_i - \Exp_{\z_i}[\epsilon_i| \calF_{i-1}] 
\end{equation}
since $\epsilon_i$ is $\calF_{i-1}$-measurable.
Putting the definition and the result of Lemma \ref{lm:strongly_convex_expansion}, we have
\begin{align}
& \norm{\theta^{t+1} - \theta^*}{2}^2  \nonumber \\
& \leq  \frac{\gamma(\gamma-1)}{(t+\gamma)(t+\gamma-1)} \norm{\theta^1 - \theta^*}{2}^2
+ \frac{\sum_{i=1}^{t}  (i+\gamma-1) \inprod{\epsilon_i}{\theta^i - \theta^*} }{\tau_\ell(t+\gamma)(t+\gamma-1)} 
+  \frac{\sum_{i=1}^{t} \norm{ \epsilon_i }{2}^2}{2\tau_\ell^2(i+\gamma)(i+\gamma-1)} \nonumber \\
& \leq \frac{\gamma(\gamma-1)}{(t+\gamma)(t+\gamma-1)} \norm{\theta^1 - \theta^*}{2}^2
+ \frac{\sum_{i=1}^{t}  (i+\gamma-1) \inprod{\epsilon_i^v}{\theta^i - \theta^*} }{\tau_\ell(t+\gamma)(t+\gamma-1)}
+ \frac{\sum_{i=1}^{t}  (i+\gamma-1) \inprod{\epsilon_i^b}{\theta^i - \theta^*} }{\tau_\ell(t+\gamma)(t+\gamma-1)} 
\nonumber \\
& +  \frac{2\sum_{i=1}^{t} \left( \norm{ \epsilon_i^v }{2}^2 - \Exp_{\z_i}[\norm{ \epsilon_i^v }{2}^2 | \calF_{i-1}] \right) }{\tau_\ell^2(t+\gamma)(t+\gamma-1)}
+  \frac{2\sum_{i=1}^{t} \Exp_{\z_i}[\norm{ \epsilon_i^v }{2}^2 | \calF_{i-1}]}{\tau_\ell^2(t+\gamma)(t+\gamma-1)}
+  \frac{2\sum_{i=1}^{t} \norm{ \epsilon_i^b }{2}^2}{\tau_\ell^2(t+\gamma)(t+\gamma-1)}.
\label{eqn:unrolled_ineq}
\end{align}
where the inequality follows from $\norm{a+b}{2}^2 \leq 2 \norm{a}{2}^2+2 \norm{a}{2}^2$. The rest of the proof is based on the analysis of inequality \eqref{eqn:unrolled_ineq}. To bound it, we first introduce the Freedman's inequality for martingale differences. The following version of Freedman's inequality is in Theorem 1.2A in \citet{victor1999general}.

\begin{lemma}
\label{lm:freedman}
(Freedman's inequality) Let $d_1, d_2, \cdots,d_T$ be a martingale difference sequence with a uniform bound $b$ on the steps $d_i$. Let $V_s$ denote the sum of conditional variances, i.e. 
$$
V_s = \sum_{i=1}^s \var(d_i | d_1,\cdots,d_{i-1}).
$$
Then, for every $a,v > 0$,
$$
\Pr \left(\sum_{i=1}^s d_i \geq a \ \text{ and } \ V_s \leq v \ \text{ for some }\ s \leq T \right)
\leq \exp \left(\frac{-a^2}{2(v+ba)} \right).
$$
\end{lemma}

Next, to apply Freedman's inequality to bound the martingale difference sequence, e.g. $\inprod{\epsilon_i^v}{\theta^i - \theta^*}$ in the second term in Eq. \eqref{eqn:unrolled_ineq}, we should control the conditional variance $\Exp_{\z_i}[\norm{\epsilon_i^v}{2}^2 | \calF_{i-1}]$ and the upper bound of L2-norm $\norm{\epsilon_i^v}{2}$. Also, as in the third and sixth term in Eq. \eqref{eqn:unrolled_ineq}, we should control the magnitude of the bias term, $\norm{\epsilon_i^v}{2}$ for all $0 \leq i \leq t$. We introduce the following lemma to control these noise terms:

\begin{lemma}
\label{lm:noise_terms_ub}
(Lemma F.5, ~\citep{gorbunov2020stochastic} )
For any $i=0,2,..,t$, we have
\begin{equation}
\label{eqn:v_norm_ub}
\norm{\epsilon_i^v}{2}
\leq  2\lambda.
\end{equation}
Moreover, for all $i=1,2,..,N$, assume that the variance of stochastic gradients is bounded by $\sigma_i^2$, i.e. $\Exp_{\z_i}[\norm{\nabla \loss(\theta^i, \z_i) - \nabla \calR(\theta^i)}{2}^2 | \calF_{i-1}] \leq \sigma_i^2 $ and assume that the norm of the true gradient is less than $\lambda/2$, i.e.
$ \norm{\nabla \calR(\theta^i)}{2} \leq \lambda/2$. Then we have
\begin{equation}
\label{eqn:b_norm_ub}
\norm{\epsilon_i^b}{2}
\leq \frac{4\sigma_i^2}{\lambda} , \text{and}
\end{equation}
\begin{equation}
\label{eqn:exp_var_ub}
\Exp_{\z_i}[\norm{\epsilon_i^v}{2}^2  | \calF_{i-1}] \leq 18 \sigma_i^2  \ \ 
\text{ for all } i=1,2,...,N.
\end{equation}
\end{lemma}
Also, recall that we have an assumption about the variance of stochastic gradient in Eq. \eqref{eqn:asp_variance_sg}:
there exist $\alpha(P,\loss)$ and $\beta(P,\loss)$ such that for every $\theta \in \Theta$:
$$
\Exp_{\z \sim P} [\norm{\nabla \loss(\theta,\z) - \nabla \calR(\theta)}{2}^2] \leq \alpha(P, \loss)\norm{\theta-\theta^*}{2}^2 + \beta(P, \loss).
$$

For brevity, in the rest of the proof, let $\alpha = \alpha(P, \loss)$ and $\beta = \beta(P, \loss)$. Also, to apply Lemma \ref{lm:noise_terms_ub}, we let $\sigma_i^2 =  \alpha \norm{\theta^i - \theta^*}{2}^2 + \beta$. Then, we have 

$$
\Rightarrow \Exp_{\z_i}[\norm{\nabla \loss(\theta^i, z_i) - \nabla \calR(\theta^i)}{2}^2 | \calF_{i-1}] 
\leq \sigma_i^2 =  \alpha \norm{\theta^i - \theta^*}{2}^2 + \beta.
$$

Now, with these two lemmas in hands, we start to analyze Eq. \eqref{eqn:unrolled_ineq}.
We first define a new constant $A$ and $C$:
\begin{equation}
\label{eqn:def_A}
A =  C_1^2 \left(
\gamma(\gamma-1)\norm{\theta^1 - \theta^*}{2}^2
+ \frac{(N+\gamma)\beta\log(2/\delta)}{\tau_\ell^2}  \right) 
\ \ \text{ and } \ \ 
C = 5000.
\end{equation}
Recall that $C_1 \geq 1$ is a scaling constant in Theorem \ref{thm:clipped_sgd}. We note that the clipping level $\lambda$ can be written in the following form:
\begin{equation}
\label{eqn:def_lambda}
\lambda = \frac{\tau_\ell\sqrt{A}}{ \log(2/\delta)}.
\end{equation}

Then we introduce new random variables: for $1 \leq i \leq N$.
\begin{equation}
\label{eqn:def_zeta}
\zeta_i = 
\begin{cases}
\theta^i - \theta^* & \text{, if } \norm{\theta^i - \theta^* }{2}^2 \leq \frac{CA}{(i+\gamma-1)(i+\gamma-2)}. \\
0 & \text{, otherwise. }
\end{cases}
\end{equation}
We note that these random variables are bounded almost surely, i.e. 
$$
\Pr \left(\norm{\zeta_i}{2} \leq  \sqrt{\frac{CA}{(i+\gamma-1)(i+\gamma-2)}} \ \right) = 1.
$$
Next, we introduce the following claim to control these two martingale difference sequences:$\{(i+\gamma-1) \inprod{\epsilon_i^v}{\zeta_i}\}_{ 1 \leq i \leq N}$ and $\{\norm{ \epsilon_i^v }{2}^2 - \Exp[\norm{\epsilon_i^v}{2}^2 | \calF_{i-1}]\}_{ 1 \leq i \leq N}$, which appeared in the second and forth terms in Eq.\eqref{eqn:unrolled_ineq} respectively.

\begin{claim}
\label{clm:freedman}
Define 
$
X_i = (i+\gamma-1) \inprod{\epsilon_i^v}{\zeta_i}
\ \text{ and }  \ 
Y_i = \norm{ \epsilon_i^v }{2}^2 - \Exp_{\z_i}[\norm{\epsilon_i^v}{2}^2 | \calF_{i-1}]$
for $1 \leq i \leq N$ be two sequence and let $$
v_i =  \var[X_i | \calF_{i-1}]
\ \text{ and }  \ 
u_i =  \var[Y_i | \calF_{i-1}]
\ \ \text{ for }  1 \leq i \leq N
$$
be its conditionally variances.
Then with probability at least $1 - \delta$, the following event holds: for all $1 \leq s \leq N$,
\begin{equation}
\label{eqn:bound_circled1}
\sum_{i=1}^s X_i < 20 \tau_{\ell}A\sqrt{C} 
+ 2A\tau_\ell C^{3/4}
\ \ \text{or} \ \  \sum_{i=1}^s v_i > \frac{36A^2 C \tau_\ell^2}{\log(2/\delta)} + \frac{3A^2 C^{3/2} \tau_\ell^2}{4\log(2/\delta)},
\end{equation}
and
\begin{equation}
\label{eqn:bound_circled3}
\sum_{i=1}^s Y_i <  118 \tau_{\ell}^2 A + 15 A\tau_\ell^2 C^{1/4} 
\ \ \text{or} \ \  \sum_{i=1}^s u_i >\frac{108 A^2 \tau_\ell^4\sqrt{C}}{\log(2/\delta)} +  \frac{5184  \tau_\ell^4 A^2}{\log(2/\delta)}.
\end{equation}
\end{claim}

We will explain the choice of these parameters in Claim \ref{clm:freedman} later.  Now, we denote $E$ be the event that Eq. \eqref{eqn:bound_circled1} and \eqref{eqn:bound_circled3} holds for all $1 \leq s \leq N$. We note that by Claim \ref{clm:freedman}, $E$ holds with probability $1-\delta$, i.e.
$$
\Pr(E) \geq 1- \delta.
$$ 
Then we prove that if $E$ holds,
\begin{equation}
\label{eqn:induction_hypothesis}
\norm{\theta^t - \theta^*}{2}^2
\leq \frac{AC}{(t+\gamma-1)(t+\gamma-2)},
\end{equation}
for $1 \leq t \leq N+1$ by induction.

First of all, we prove that it holds for $t=1$. 
From our definition of constant $A$ in Eq.\eqref{eqn:def_A} and the fact that $C_1 \geq 1$, this case holds trivially, i.e.
$$
\norm{\theta^1 - \theta^*}{2}^2
\leq \frac{C}{\gamma(\gamma-1)}\left(
\gamma(\gamma-1)\norm{\theta^1 - \theta^*}{2}^2
+ \frac{(N+\gamma)\beta\log(2/\delta)}{\tau_\ell^2}  \right) = \frac{AC}{\gamma(\gamma-1)}.
$$
Next, we assume that Eq. \eqref{eqn:induction_hypothesis} holds for $t=1,\cdots,n$. When $t=n+1$, by Eq. \eqref{eqn:unrolled_ineq}, we have
\begin{align}
& (n+\gamma)(n+\gamma-1)\norm{\theta^{n+1} - \theta^*}{2}^2 \nonumber \\
& \leq
\underbrace{\gamma(\gamma-1) \norm{\theta^1 - \theta^*}{2}^2}_{\circledOne}
+ \underbrace{\frac{\sum_{i=1}^{n}  (i+\gamma-1) \inprod{\epsilon_i^v}{\theta^i - \theta^*} }{\tau_\ell}}_{\circledTwo}
+ \underbrace{\frac{\sum_{i=1}^{n}  (i+\gamma-1) \inprod{\epsilon_i^b}{\theta^i - \theta^*} }{\tau_\ell} }_{\circledThree}
\nonumber \\
& + \underbrace{ \frac{2\sum_{i=1}^{n} \norm{ \epsilon_i^v }{2}^2 - \Exp_{z_i}[\norm{ \epsilon_i^v }{2}^2 | \calF_{i-1}] }{\tau_\ell^2}}_{\circledFour}
+ \underbrace{ \frac{2\sum_{i=1}^{n} \Exp_{z_i}[\norm{ \epsilon_i^v }{2}^2 | \calF_{i-1}]}{\tau_\ell^2}}_{\circledFive}
+  \underbrace{\frac{2\sum_{i=1}^{n} \norm{ \epsilon_i^b }{2}^2}{\tau_\ell^2}}_{\circledSix}.
\label{eqn:unrolled_ineq2}
\end{align}

\textbf{Check conditions in Lemma \ref{lm:noise_terms_ub}:}
Before we upper bound $\circledOne \sim \circledSix$, we first prove that $ \norm{\nabla \calR(\theta^t)}{2} \leq 2 \lambda$. Since $\calR(\cdot)$ is $\tau_u$-smooth by assumption in Eq.\ref{eqn:convexity_parameter}, or its gradient is $\tau_u$-Lipschitz, for $1 \leq t \leq n$ we have 
\begin{align}
\norm{\nabla \calR(\theta^t)}{2}
 \leq \tau_u\norm{\theta^t - \theta^*}{2} 
& \overset{\eqref{eqn:induction_hypothesis}}{\leq} \tau_u  \sqrt{\frac{AC}{(t+\gamma-1)(t+\gamma-2)}} \nonumber \\
& \overset{t \geq 1}{\leq }
\tau_u  \sqrt{\frac{AC}{\gamma(\gamma-1)}} \nonumber \\
& \leq \frac{\tau_u}{\gamma-1}\sqrt{AC} \nonumber \\
& \overset{\eqref{eqn:def_hyperparemeter}}{\leq} \frac{2\tau_\ell C_1 \sqrt{A}}{\log(2/\delta)} = 2\lambda,
\end{align}
where the last inequality is due to the definition of $\gamma -1  > \frac{\sqrt{C}}{2}  \frac{\tau_u}{\tau_\ell} \log(2/\delta)$.
Therefore, the condition in Lemma \ref{lm:noise_terms_ub} holds for $1 \leq t \leq n$, which means we could use Eq.\eqref{eqn:b_norm_ub} and \eqref{eqn:exp_var_ub} to control the bias and variance terms for $ 1 \leq t \leq n$.
 
\paragraph{Upper bounds for $\circledOne$:}
\begin{equation}
\label{eqn:final_bound_circledOne}
\circledOne \leq C_1^2 \left( \gamma(\gamma-1)\norm{\theta^1 - \theta^*}{2}^2
+ \frac{(N+\gamma)\beta\log(2/\delta)}{\tau_\ell^2} \right) = A.
\end{equation}

\paragraph{Upper bounds for $\circledTwo$:}
Recall that our definition of bounded variable $\zeta_i$ in Eq. \eqref{eqn:def_zeta} and martingale difference sequence $\{X_i\}_{1\leq i \leq n}$ in Claim \ref{clm:freedman}. Under the induction hypothesis, we have 
$$
\zeta_i = \theta^i - \theta^* 
\ \ \text{ for $1 \leq i \leq n$ and } \ \ 
\circledTwo = \frac{\sum_{i=1}^n X_i}{\tau_{\ell}}.
$$
We first show that the sum of its conditional variances are upper bounded, i.e.
$$
\sum_{i=1}^n \var[X_i^2 | \calF_{i-1}] = \sum_{i=1}^n v_i \leq \frac{36A^2 C \tau_\ell^2}{\log(2/\delta)}
+ \frac{3A^2 C^{3/2} \tau_\ell^2}{4\log(2/\delta)}.
$$
Since $X_i = (i+\gamma-1)\inprod{\epsilon_i^v}{\zeta_i}$ for $1 \leq i \leq n$, we have 
\begin{align}
\sum_{i=1}^n v_i 
& \leq \sum_{i=1}^n  (i+\gamma-1)^2  \Exp[\norm{\zeta_i}{2}^2\norm{\epsilon_i^v}{2}^2| \calF_{i-1}] 
\nonumber \\
& \overset{\eqref{eqn:def_zeta}}{\leq} \sum_{i=1}^n  \frac{(i+\gamma-1)^2CA}{(i+\gamma-1)(i+\gamma-2)} \Exp[\norm{\epsilon_i^v}{2}^2| \calF_{i-1}] 
\nonumber \\
& \leq 2CA \sum_{i=1}^n \Exp[\norm{\epsilon_i^v}{2}^2 | \calF_{i-1}] \nonumber \\
& \overset{\eqref{eqn:v_norm_ub},\eqref{eqn:asp_variance_sg}}{\leq} 
2CA \sum_{i=1}^n 18 (\alpha \norm{\theta^i - \theta^*}{2}^2 + \beta) \nonumber \\
& \overset{\eqref{eqn:induction_hypothesis}}{\leq} 36CA \left( n \beta +  \alpha\sum_{i=1}^n \frac{AC}{(i+\gamma-1)(i+\gamma-2)}
 \right)
\nonumber \\
& \leq 36CA \left( n \beta +  \alpha AC \sum_{i=1}^n \left(\frac{1}{i+\gamma-2} - \frac{1}{i+\gamma-1 } \right)
 \right)
\nonumber \\
& \leq 36CA \left( n \beta +  \frac{ \alpha AC }{\gamma-1}\right)
\nonumber \\
& \overset{n \leq N}{\leq} 36\beta NAC +  \frac{ 36\alpha A^2C^2 }{\gamma-1}.
\end{align}
Since we have $A \overset{\eqref{eqn:def_A}}{\geq} \frac{(N+\gamma)\beta\log(2/\delta)}{\tau_\ell^2} \geq \frac{\beta N \log(2/\delta)} {\tau_\ell^2} $ and $\gamma -1 \overset{\eqref{eqn:def_hyperparemeter}}{\geq} \frac{48\alpha \sqrt{C}\log(2/\delta)}{\tau_\ell^2}$, we have 
\begin{align}
\sum_{i=1}^n v_i 
& \leq 36(\beta N) AC +  \frac{ 36\alpha A^2C^2 }{\gamma-1} \nonumber \\
& \leq \frac{36A^2 C \tau_\ell^2}{\log(2/\delta)}
+ \frac{3A^2 C^{3/2} \tau_\ell^2}{4\log(2/\delta)}.
\end{align}

Therefore, we know the second inequality in Eq. \eqref{eqn:bound_circled1} does not hold, so the first one must hold for $t=n$, i.e. 
$$
\sum_{i=1}^n X_i < 20 \tau_{\ell}A\sqrt{C} 
+ 2A\tau_\ell C^{3/4}.
$$
Then we can upper bound $\circledTwo$:
\begin{align}
\circledTwo 
 = \frac{\sum_{i=1}^n X_i}{\tau_{\ell}} 
\leq A(20\sqrt{C} + 2C^{3/4})
\label{eqn:final_bound_circledTwo}
\end{align}

\paragraph{Upper bound $\circledThree$:}
By Eq.\eqref{eqn:b_norm_ub} in Lemma \ref{lm:noise_terms_ub}, we have $\norm{\epsilon_t^b}{2} \leq \frac{4(\alpha \norm{\theta^t -\theta^*}{2}^2 + \beta)}{\lambda}$. Then we have 
\begin{align}
\circledThree 
& = \frac{\sum_{i=1}^{n}  (i+\gamma-1) \inprod{\epsilon_i^b}{\theta^i - \theta^*} }{\tau_\ell} \nonumber \\
& \leq \frac{\sum_{i=1}^{n}  (i+\gamma-1) \norm{\epsilon_i^b}{2} \norm{\theta^i - \theta^*}{2} }{\tau_\ell} \nonumber \\
& \overset{\eqref{eqn:b_norm_ub}}{\leq}
\frac{\sum_{i=1}^{n} 4(i+\gamma-1)( \alpha \norm{\theta^i - \theta^*}{2}^3 + \beta\norm{\theta^i - \theta^*}{2}) }{\lambda \tau_\ell} \nonumber \\
& \overset{\eqref{eqn:def_lambda},\eqref{eqn:induction_hypothesis}}{\leq}
\frac{\sum_{i=1}^{n} 4(i+\gamma-1) \sqrt{\frac{CA}{(i+\gamma-1)(i+\gamma-2)}}( \frac{CA\alpha}{(i+\gamma-1)(i+\gamma-2)} + \beta) \log(2/\delta)}{ \tau_\ell^2 \sqrt{A}} 
\nonumber \\
& \leq \frac{\sum_{i=1}^{n} 8 \sqrt{CA}( \frac{CA\alpha}{(i+\gamma-1)(i+\gamma-2)} + \beta) \log(2/\delta)}{ \tau_\ell^2 \sqrt{A}} 
\nonumber \\
& \leq 
\frac{\sum_{i=1}^{n} 8 \beta \sqrt{C}\log(2/\delta) +   8\alpha AC  \sqrt{C}\log(2/\delta) \left( \frac{1}{(i+\gamma-2)}  - \frac{1}{(i+\gamma-1)}\right) }{ \tau_\ell^2 }
\nonumber \\
& \overset{n \leq N}{\leq} 
\frac{ 8 N\beta \sqrt{C}\log(2/\delta)}{ \tau_\ell^2 } + \frac{8\alpha AC  \sqrt{C}\log(2/\delta)}{\tau_\ell^2(\gamma -1) }.
\end{align}
Next, since 
$$
A \geq \frac{(N+\gamma)\beta\log(2/\delta)}{\tau_\ell^2} \geq \frac{N\beta\log(2/\delta)}{\tau_\ell^2} 
\ \ \text{ and } \ \
\gamma -1 \overset{\eqref{eqn:def_hyperparemeter}}{\geq} \frac{48\alpha \sqrt{C} \log(2/\delta)}{\tau_\ell^2},
$$
We have 
\begin{equation}
\label{eqn:final_bound_circledThree}
\circledThree 
\leq 8\sqrt{C}\frac{  N\beta \log(2/\delta)}{ \tau_\ell^2 } + \frac{AC}{6} \times \frac{48\alpha \sqrt{C}\log(2/\delta)}{\tau_\ell^2(\gamma -1) } 
\leq A \left(  8 \sqrt{C}+ \frac{C}{6} \right).
\end{equation}

\paragraph{Upper bound of $\circledFour$:}
Recall that our definition of martingale difference sequence $\{Y_i\}_{1\leq i \leq n}$ in Claim \ref{clm:freedman}. We have
$$
\circledFour
= \frac{2\sum_{i=1}^{n} \left(\norm{ \epsilon_i^v }{2}^2 - \Exp_{z_i}[\norm{ \epsilon_i^v }{2}^2 | \calF_{i-1}]  \right) }{\tau_\ell^2}
= \frac{2\sum_{i=1}^n Y_i}{\tau_{\ell}^2}.
$$
We first show that the sum of its conditional variances is upper bounded, i.e.
$$
\sum_{i=1}^n \var[Y_i^2 | \calF_{i-1}] = \sum_{i=1}^n u_i \leq \frac{108 A^2 \tau_\ell^4\sqrt{C}}{\log(2/\delta)} +  \frac{5184  \tau_\ell^4 A^2}{\log(2/\delta)}.
$$
Since we have $\norm{\epsilon_i^v}{2} \leq 2 \lambda$ by Eq. \eqref{eqn:v_norm_ub}, we get $\left| \norm{ \epsilon_i^v }{2}^2 - \Exp_{z_i}[\norm{ \epsilon_i^v }{2}^2 | \calF_{i-1}] \right| \overset{\eqref{eqn:v_norm_ub}}{\leq} 4\lambda^2 + 4\lambda^2 = 8\lambda^2$. Then,
\begin{align}
\sum_{i=1}^n u_i 
& \leq \sum_{i=1}^n \Exp_{z_i} \left[ \left(\norm{ \epsilon_i^v }{2}^2 - \Exp_{z_i}[\norm{ \epsilon_i^v }{2}^2 | \calF_{i-1}] \right)^2 | \calF_{i-1} \right] \nonumber \\
& \leq \sum_{i=1}^n 8\lambda^2\Exp_{z_i} \left[ \left| \norm{ \epsilon_i^v }{2}^2 - \Exp_{z_i}[\norm{ \epsilon_i^v }{2}^2 | \calF_{i-1}] \right| | \calF_{i-1} \right] \nonumber \\
& \leq \sum_{i=1}^n 8\lambda^2 \times 2\Exp_{z_i}[\norm{ \epsilon_i^v }{2}^2 | \calF_{i-1}] \nonumber \\
& \overset{\eqref{eqn:exp_var_ub}}{\leq} 
16\lambda^2 \sum_{i=1}^n 18(\alpha\norm{\theta^i - \theta^*}{2}^2 + \beta)\nonumber \\
& \overset{\eqref{eqn:def_lambda},\eqref{eqn:induction_hypothesis}}{\leq} 
\frac{288\tau_\ell^2 A}{\log(2/\delta)^2} \sum_{i=1}^n 18 \left(\alpha AC \left(\frac{1}{i+\gamma-2}- \frac{1}{i+\gamma-1} \right) + \beta \right)\nonumber \\
& \overset{n \leq N}{\leq} \frac{5184 \alpha \tau_\ell^2 A^2C}{\log(2/\delta)^2(\gamma-1)}  + \frac{5184 N\beta \tau_\ell^2 A}{\log(2/\delta)^2}.
\end{align}

Next, since $\delta \leq 2e^{-1}$ and
$$
A \geq \frac{(N+\gamma)\beta\log(2/\delta)}{\tau_\ell^2} \geq \frac{N\beta}{\tau_\ell^2} 
\ \ \text{ and } \ \
\gamma -1 \overset{\eqref{eqn:def_hyperparemeter}}{\geq} \frac{48\alpha \sqrt{C} \log(2/\delta)}{\tau_\ell^2},
$$
We have 
\begin{align*}
\sum_{i=1}^n u_i
& \leq \frac{5184 \alpha \tau_\ell^2 A^2C}{\log(2/\delta)^2(\gamma-1)}  + \frac{5184 N\beta \tau_\ell^2 A}{\log(2/\delta)^2} \\
& \overset{\delta \leq 2e^{-1}}{\leq} \frac{108 A^2 \tau_\ell^4\sqrt{C}}{\log(2/\delta)} +  \frac{5184  \tau_\ell^4 A^2}{\log(2/\delta)}.
\end{align*}

Therefore, we know that the second inequality in Eq. \eqref{eqn:bound_circled3} does not hold, so that the first one must be satisfied, i.e.  
$$
\sum_{i=1}^n Y_i < 118 \tau_{\ell}^2 A + 15 A\tau_\ell^2 C^{1/4}.
$$
Finally, we have 
\begin{equation}
\label{eqn:final_bound_circledFour}
\circledFour = \frac{2\sum_{i=1}^n Y_i}{\tau_\ell^2} \leq  A( 236 + 30 C^{1/4} ).
\end{equation}

\paragraph{Upper bound of $\circledFive $:}
By equation Eq.\eqref{eqn:exp_var_ub} in Lemma \ref{lm:noise_terms_ub}, we have $\Exp_{z_t} [\norm{\epsilon_t^v}{2}^2 | \calF_{t-1} ] \leq 18(\alpha \norm{\theta^t -\theta^*}{2}^2 + \beta)$. Then we have 
\begin{align}
\circledFive 
& = \frac{2\sum_{i=1}^{n} \Exp_{z_i}[\norm{ \epsilon_i^v }{2}^2 | \calF_{i-1}]}{\tau_\ell^2} \nonumber \\ 
& \overset{\eqref{eqn:exp_var_ub},\eqref{eqn:induction_hypothesis}}{\leq} 
\frac{36n\beta}{\tau_\ell^2} + 
\frac{36\alpha AC \sum_{i=1}^{n} \left(\frac{1}{i+\gamma -2 } - \frac{1}{i+\gamma -1 }  \right)}{\tau_\ell^2}
\nonumber \\ 
& \overset{n\leq N}{\leq} \frac{36N\beta}{\tau_\ell^2} + 
\frac{36\alpha AC }{(\gamma-1)\tau_\ell^2}.
\end{align}
Next, since 
$$
A \geq \frac{(N+\gamma)\beta\log(2/\delta)}{\tau_\ell^2} \overset{\delta \leq 2e^{-1}}{\geq} \frac{N\beta\log(2/\delta)}{\tau_\ell^2}
\ \ \text{ and } \ \
\gamma -1 \geq \frac{48\alpha \sqrt{C} \log(2/\delta)}{\tau_\ell^2} \overset{\delta \leq 2e^{-1}}{\geq} \frac{48\alpha \sqrt{C}}{\tau_\ell^2}.
$$
We have 
\begin{equation}
\label{eqn:final_bound_circledFive}
\circledFive \leq A(  36 + \frac{3\sqrt{C}}{4} ) \leq A(  36 + \sqrt{C} ).
\end{equation}

\paragraph{Upper bound of $\circledSix $:}
By equation Eq.\eqref{eqn:b_norm_ub} in Lemma \ref{lm:noise_terms_ub}, we have $\norm{\epsilon_t^b}{2} \leq \frac{4(\alpha \norm{\theta^t -\theta^*}{2}^2 + \beta)}{\lambda}$. Then we have 
\begin{align}
\circledSix
& = \frac{2\sum_{i=1}^{n}\norm{ \epsilon_i^b }{2}^2}{\tau_\ell^2}  \nonumber \\ 
& \overset{\eqref{eqn:b_norm_ub}}{\leq} 
\frac{32 \sum_{i=1}^{n} \left( \beta + \alpha \norm{\theta^i-\theta^*}{2}^2 \right)^2}{\lambda^2\tau_\ell^2}  \nonumber \\ 
& \leq 
\frac{64 \sum_{i=1}^{n} \left( \beta^2 + \alpha^2 \norm{\theta^i-\theta^*}{2}^4 \right)}{\lambda^2\tau_\ell^2} \ \ (\text{Cauchy-Schwarz})  \nonumber \\ 
& \overset{n\leq N, \eqref{eqn:induction_hypothesis}}{\leq} 
\frac{64N\beta^2 + 64A^2C^2\alpha^2\sum_{i=1}^{n} \frac{1}{(i+\gamma-2)^2(i+\gamma-1)^2}}{\lambda^2\tau_\ell^2}  \nonumber \\ 
& \overset{(i)}{\leq} 
\frac{64N\beta^2 + \frac{64A^2C^2\alpha^2}{(\gamma-1)^2}}{\lambda^2\tau_\ell^2}  \nonumber \\ 
& \overset{\eqref{eqn:def_lambda}}{=} 
\frac{64N\beta^2\log(2/\delta)^2 }{A\tau_\ell^4} +\frac{64AC^2\alpha^2\log(2/\delta)^2}{\tau_\ell^4(\gamma-1)^2},
\end{align}
where $(i)$ is due to 
\begin{align*}
\sum_{i=1}^{n} \frac{1}{(i+\gamma-2)^2(i+\gamma-1)^2}
& \leq \left( \sum_{i=1}^{n} \frac{1}{(i+\gamma-2)(i+\gamma-1)} \right)^2 \\
& \leq \left( \sum_{i=1}^{n} \frac{1}{i+\gamma-2}  - \frac{1}{i+\gamma-1} \right)^2 \\
& \leq \frac{1}{(\gamma-1)^2}.
\end{align*}

$$
A \geq \frac{(N+\gamma)\beta\log(2/\delta)}{\tau_\ell^2} \overset{\delta \leq 2e^{-1}}{\geq} \frac{N\beta\log(2/\delta)}{\tau_\ell^2}
\ \ \text{ and } \ \
\gamma -1 \geq \frac{48\alpha \sqrt{C} \log(2/\delta)}{\tau_\ell^2}.
$$
We have 
\begin{align}
\circledSix & \leq \frac{64 }{AN} \times \frac{N^2\beta^2\log(2/\delta)^2}{\tau_\ell^4} +\frac{64AC^2\alpha^2\log(2/\delta)^2}{\tau_\ell^4(\gamma-1)^2} \nonumber \\
& \leq 
A( \frac{64}{N} +  \frac{C}{36} )
\overset{N \geq 1}{\leq} A( 64 +  \frac{C}{36} ).
\label{eqn:final_bound_circledSix}
\end{align}

\paragraph{Upper bounds for $\circledOne + \circledTwo + \circledThree + \circledFour + \circledFive + \circledSix$:}
Now, we have derived the upper bounds for $\circledOne \sim \circledSix$. By combining Eq. \eqref{eqn:final_bound_circledOne},\eqref{eqn:final_bound_circledTwo},\eqref{eqn:final_bound_circledThree},\eqref{eqn:final_bound_circledFour},\eqref{eqn:final_bound_circledFive},\eqref{eqn:final_bound_circledSix}, we have

\begin{align}
 (n+\gamma)(n+\gamma-1)\norm{\theta^{n+1} - \theta^*}{2}^2 
& \leq \circledOne + \circledTwo + \circledThree + \circledFour + \circledFive + \circledSix 
\nonumber \\
& \leq A(1 + 20\sqrt{C} + 2C^{3/4} + 8\sqrt{C} + \frac{C}{6} + 236 + 30C^{1/4} +36 + \sqrt{C} +64 + \frac{C}{36}  ) \nonumber \\
& = A(337 + 30C^{1/4} + 29\sqrt{C} + 2C^{3/4} + \frac{7C}{36} )
\nonumber \\
& \leq AC,
\end{align}
where the last inequality follows from the definition of $C=5000$. Therefore, we can conclude our induction proof. 

Lastly, by plugging in $n=N$ and the definition of $A$ in the above equation, we have, with probability at least $1-\delta$ (so that $E$ holds), 
\begin{align}
\norm{\theta^{N+1} -\theta^*}{2}^2
& \leq \frac{C C_1^2\left(
\gamma(\gamma-1)\norm{\theta^1 - \theta^*}{2}^2
+ \frac{(N+\gamma)\beta\log(2/\delta)}{\tau_\ell^2}  \right)}{(N+\gamma)(N+\gamma-1)} \nonumber \\
& \leq 
2C C_1^2\left(\frac{\gamma^2\norm{\theta^1 - \theta^*}{2}^2}{(N+\gamma)^2}
+ \frac{\beta(P,\loss)\log(2/\delta)}{(N+\gamma)\tau_\ell^2} \right),
\end{align}
where the last inequality is due to $(N+\gamma-1) \geq \frac{N+\gamma}{2}$.
If we take square root on the both sides and use the inequality $\sqrt{a+b} \leq \sqrt{a} + \sqrt{b}$ for $a,b \in \real^+$,  we have

\begin{align}
\norm{\theta^{N+1} -\theta^*}{2}
& \leq 100 C_1\left( \frac{\gamma \norm{\theta^1 - \theta^*}{2}}{N+\gamma}
+ \frac{1}{\tau_\ell}\sqrt{\frac{\beta(P,\loss)\log(2/\delta)}{N+\gamma}} \right).
\end{align}

\end{proof}

\subsection{Proof of Claim \ref{clm:freedman}}
\begin{proof}

\textbf{(Bounds for the first martingale difference sequence):} for martingale difference sequence $\{(i+\gamma-1) \inprod{\epsilon_i^v}{\zeta_i}\}_{ 0 \leq i \leq T-1}$,

(i) we first check that it is conditionally unbiased, i.e. 
\begin{align*}
\Exp_{z_i \sim P} [\epsilon_i^v | \calF_{i-1}] = 0 \Rightarrow
\Exp_{z_i \sim P} [(i+\gamma-1) \inprod{\epsilon_i^v}{\zeta_i} | \calF_{i-1}] = 0,
\end{align*}
since $\zeta_i$ is determinant conditioned on $\calF_{i-1}$.

(ii) We check each summand is bounded, i.e.
$$
\norm{(i+\gamma-1) \inprod{\epsilon_i^v}{\zeta_i}}{2}
\overset{\eqref{eqn:v_norm_ub},\eqref{eqn:def_zeta}}{\leq} (i+\gamma-1)(2\lambda)\sqrt{\frac{CA}{(i+\gamma-1)(i+\gamma-2)}} \leq 4\lambda\sqrt{A}
= \frac{4 A \tau_\ell \sqrt{C}}{\log(2/\delta)}
$$
for $1 \leq i \leq N$.

(iii) Let  $V_s = \sum_{i=1}^s v_i$ be the sum of its conditional variance.
Then we apply the Freedman's inequality in Lemma \ref{lm:freedman} instantiated with the following parameters 
$$
b =\frac{4 A \tau_\ell \sqrt{C}}{\log(2/\delta)}
\ , \ \  
v = \frac{36A^2 C \tau_\ell^2}{\log(2/\delta)} + \frac{3A^2 C^{3/2} \tau_\ell^2}{4\log(2/\delta)}
\text{ and } 
a = 2b\log(2/\delta) + \sqrt{2v\log(2/\delta)}.
$$
We will specify our choices of parameters later. Then we have 
\begin{align}
\Pr \left( \sum_{i=1}^{s}  (i+\gamma-1) \inprod{\epsilon_i^v}{\zeta_i} \geq a
\ \text{ and } \ 
V_s \leq v
\text{ for some } s \leq N \right) 
\leq \exp \left(\frac{-a^2}{2(v+ba)} \right) 
\leq \frac{\delta}{2},
\label{eqn:bound_first_martingale sequence}
\end{align}
where the last inequality is due to 
\begin{align}
\exp \left(\frac{-a^2}{2(v+ba)} \right) \leq \frac{\delta}{2}
& \Leftrightarrow \frac{-a^2}{2(v+ba)} \leq \log(\frac{\delta}{2}) \nonumber \\
& \Leftrightarrow
a^2 - 2b\log(\frac{\delta}{2})a - 2v\log \left(\frac{\delta}{2} \right) \geq 0 \nonumber \\
& \Leftrightarrow a \geq b\log(\frac{\delta}{2}) + \sqrt{(b\log(\frac{\delta}{2}))^2 + 2v\log(\frac{\delta}{2})}.
\label{eqn:a_to_delta}
\end{align}
The choice of $a$ satisfies the above inequality due to the fact that $\sqrt{x+y} \leq \sqrt{x} + \sqrt{y}$ for $x,y \in \real^+ $.
Also, we have 
\begin{align}
a & = 2b\log(2/\delta) + \sqrt{2v\log(2/\delta)} \nonumber \\
& =   8A \tau_\ell \sqrt{C} +  \sqrt{72A^2 C \tau_\ell^2+ \frac{3A^2 C^{3/2} \tau_\ell^2}{2}} \nonumber \\
& \leq 8A \tau_\ell \sqrt{C} + 12A \tau_\ell \sqrt{C} + 2AC^{3/4}\tau_\ell \nonumber \\
& = 20\tau_\ell A\sqrt{C} + 2A\tau_\ell C^{3/4}.
\end{align}
Therefore, by Eq. \eqref{eqn:bound_first_martingale sequence}, we have 
\begin{align}
& 1 - \frac{\delta}{2} \nonumber \\
& \geq  \Pr \left( \sum_{i=1}^{s}  (i+\gamma-1) \inprod{\epsilon_i^v}{\zeta_i} \geq a
\ \text{ and } \ 
V_s \leq v
\text{ for some } s \leq N \right) \nonumber \\
& = 1 - \Pr \left( \sum_{i=1}^{s}  (i+\gamma-1) \inprod{\epsilon_i^v}{\zeta_i} < a
\ \text{ or } \ 
V_s > v
\text{ for all } 1 \leq  s \leq N \right) 
\nonumber \\
& \geq 1 - \Pr \left( \sum_{i=1}^{s}  (i+\gamma-1) \inprod{\epsilon_i^v}{\zeta_i} < 20\tau_\ell A\sqrt{C} + 2A\tau_\ell C^{3/4}
\ \text{ or } \ 
V_s > \frac{36A^2 C \tau_\ell^2}{\log(2/\delta)} + \frac{3A^2 C^{3/2} \tau_\ell^2}{4\log(2/\delta)}
\text{ for all } 1 \leq  s \leq N \right).
\label{eqn:final_bound_martingale_sequence1}
\end{align}

\textbf{2.(Bounds for the second martingale difference sequence):} For martingale difference sequence $\{\norm{ \epsilon_i^v }{2}^2 - \Exp[\norm{\epsilon_i^v}{2}^2 | \calF_{i-1}]\}_{ 1 \leq i \leq N}$, (i) we first check that it is conditionally unbiased, i.e. 
$$
\Exp \left[\norm{ \epsilon_i^v }{2}^2 - \Exp[\norm{\epsilon_i^v}{2}^2 | \calF_{i-1}]  \ | \calF_{i-1} \right] = 0
$$

(ii) We check each summand is bounded, i.e.
$$
\left| \norm{ \epsilon_i^v }{2}^2 - \Exp[\norm{\epsilon_i^v}{2}^2 | \calF_{i-1}] \right|
\leq \norm{ \epsilon_i^v }{2}^2 + \Exp[\norm{\epsilon_i^v}{2}^2 | \calF_{i-1}]
\overset{\eqref{eqn:v_norm_ub}}{\leq} 
 4\lambda^2 + 4\lambda^2 
 = 8\lambda^2  = \frac{8A\tau_\ell^2}{\log^2(2/\delta)} \leq \frac{8A\tau_\ell^2}{\log(2/\delta)}
$$
for $1 \leq i \leq N$ since $\norm{\epsilon_i^v}{2} \leq 2 \lambda$.

(iii) Let $U_s = \sum_{i=1}^s u_i^2$  be the sum of its conditional variance.
Then we apply the Freedman's inequality instantiated with parameters 
$$
b =\frac{8A\tau_\ell^2}{\log(2/\delta)}
\ , \ \  
v =  \frac{108 A^2 \tau_\ell^4\sqrt{C}}{\log(2/\delta)} +  \frac{5184  \tau_\ell^4 A^2}{\log(2/\delta)}
\text{ and } 
 a = 2b\log(2/\delta) + \sqrt{2v\log(2/\delta)}.
$$
We will specify our choices of parameters later. Then we have
\begin{align}
\Pr \left( \sum_{i=1}^{s}\norm{ \epsilon_i^v }{2}^2 - \Exp[\norm{\epsilon_i^v}{2}^2  \geq a
\ \text{ and } \ 
U_{s} \leq v
\text{ for some } s \leq N \right)
\leq \exp\left(\frac{-a^2}{2(v+ba)} \right)
\overset{\eqref{eqn:a_to_delta}}{\leq} \frac{\delta}{2}
\label{eqn:bound_second_martingale sequence}
\end{align}
for all $1 \leq i \leq N$. The choice of $a$ satisfies the above inequality due to the fact that $\sqrt{x+y} \leq \sqrt{x} + \sqrt{y}$ for $x,y \in \real^+ $.
Also, we have 
\begin{align}
a & = 2b\log(2/\delta) + \sqrt{2v\log(2/\delta)} \nonumber \\
& = 8A \tau_\ell^2  +  \sqrt{ 216 A^2 \tau_\ell^4\sqrt{C} +  10368  \tau_\ell^4 A^2 } \nonumber \\
& \leq   8A \tau_\ell^2 + 15A\tau_\ell^2 C^{1/4} + 110\tau_\ell^2 A
\nonumber \\
& = 118A \tau_\ell^2 + 15A\tau_\ell^2 C^{1/4}.
\end{align}
Therefore, by Eq. \eqref{eqn:bound_second_martingale sequence}, we have 
\begin{align}
& 1 - \frac{\delta}{2}  \nonumber \\
& \geq  \Pr \left(\sum_{i=1}^{s}\norm{ \epsilon_i^v }{2}^2 - \Exp[\norm{\epsilon_i^v}{2}^2  \geq a
\ \text{ and } \ 
U_s \leq v
\text{ for some } s \leq N \right) \nonumber \\
& = 1 - \Pr \left( \sum_{i=1}^{s}\norm{ \epsilon_i^v }{2}^2 - \Exp[\norm{\epsilon_i^v}{2}^2  < a
\ \text{ or } \ 
U_s > v
\text{ for all } 1 \leq  s \leq N \right) 
\nonumber \\
& \geq 1 - \Pr \left( \sum_{i=1}^{s}\norm{ \epsilon_i^v }{2}^2 - \Exp[\norm{\epsilon_i^v}{2}^2  <  118A \tau_\ell^2 + 15A\tau_\ell^2 C^{1/4}
\ \text{ or } \ 
U_s > \frac{108 A^2 \tau_\ell^4\sqrt{C}}{\log(2/\delta)} +  \frac{5184  \tau_\ell^4 A^2}{\log(2/\delta)}
\text{ for all } 1 \leq  s \leq N \right).
\label{eqn:final_bound_martingale_sequence2}
\end{align}

Therefore, by combining Eq.\eqref{eqn:final_bound_martingale_sequence1} and \eqref{eqn:final_bound_martingale_sequence2}, we have, with probability at least $1-\delta$, for all $1 \leq s \leq N$,
\begin{equation}
\sum_{i=1}^s X_i < 20 \tau_{\ell}A\sqrt{C} 
+ 2A\tau_\ell C^{3/4}
\ \ \text{or} \ \  \sum_{i=1}^s v_i > \frac{36A^2 C \tau_\ell^2}{\log(2/\delta)} + \frac{3A^2 C^{3/2} \tau_\ell^2}{4\log(2/\delta)},
\end{equation}
and
\begin{equation}
\sum_{i=1}^s Y_i <  118 \tau_{\ell}^2 A + 15 A\tau_\ell^2 C^{1/4} 
\ \ \text{or} \ \  \sum_{i=1}^s u_i >\frac{108 A^2 \tau_\ell^4\sqrt{C}}{\log(2/\delta)} +  \frac{5184  \tau_\ell^4 A^2}{\log(2/\delta)}.
\end{equation}

\end{proof}

\newpage 
\section{Proofs of Corollaries}
\label{sec:proof_corollaries}
In this section, we provide proofs for Corollary \ref{cor:sample_complexity}, \ref{cor:heavy_tailed_mean_estimation} and \ref{cor:heavy_tailed_linear_regression}.

\subsection{Proof of Corollary \ref{cor:sample_complexity}}
Since $\calR(\cdot)$ is $\tau_\ell$-strongly convex and $\tau_u$-smooth, we have, for all $\theta \in \Theta $
\begin{equation}
\frac{\tau_\ell}{2} \norm{\theta-\theta^*}{2}^2
\leq \calR(\theta) - \calR(\theta^*)
\leq \frac{\tau_u}{2} \norm{\theta-\theta^*}{2}^2.
\end{equation}

Also, by Theorem \ref{thm:clipped_sgd}, we have 
\begin{align*}
\norm{\theta^{N+1} -\theta^*}{2}^2
& \leq O\left(\frac{\gamma^2 \norm{\theta^1 - \theta^*}{2}^2}{(N+\gamma)^2}
+ \frac{1}{\tau_\ell^2}\frac{\beta(P,\loss)\log(1/\delta)}{N+\gamma} \right) \\
& = O\left(\frac{ \tau_u^2\norm{\theta^1 - \theta^*}{2}^2\log(1/\delta)^2}{\tau_\ell^2 N^2}
+ \frac{1}{\tau_\ell^2}\frac{\sigma^2\log(1/\delta)}{N} \right).
\end{align*}
Therefore, by combining these equations, we obtain
\begin{align*}
\calR(\theta^{N+1}) - \calR(\theta^*)
& \leq \frac{\tau_u}{2} \norm{\theta^{N+1} -\theta^*}{2}^2 \\
& \leq O\left(\frac{ \tau_u^3\norm{\theta^1 - \theta^*}{2}^2\log(1/\delta)^2}{\tau_\ell^2N^2}
+ \frac{\tau_u}{\tau_\ell^2}\frac{\sigma^2\log(1/\delta)}{N} \right) \\
& = O\left(\frac{ \tau_u^3 \left( \tau_\ell\norm{\theta^1 - \theta^*}{2}^2 \right) \log(1/\delta)^2}{\tau_\ell^3 N^2}
+ \frac{\tau_u}{\tau_\ell^2}\frac{\sigma^2\log(1/\delta)}{N} \right) \\
& \leq O\left(\frac{ \tau_u^3 (\calR(\theta^1) - \calR(\theta^*)) \log(1/\delta)^2}{\tau_\ell^3 N^2}
+ \frac{\tau_u}{\tau_\ell^2}\frac{\sigma^2\log(1/\delta)}{N} \right) \\
& = O\left(\frac{ \tau_u^3 r_0 \log(1/\delta)^2}{\tau_\ell^3 N^2}
+ \frac{\tau_u}{\tau_\ell^2}\frac{\sigma^2\log(1/\delta)}{N} \right).
\end{align*}

\subsection{Proof of Corollary \ref{cor:heavy_tailed_mean_estimation}}
Since $\nabla_\theta \loss(\theta,\z) = \theta- \z$, we have 
\begin{equation}
\Exp_{z \sim P}\norm{\nabla_\theta \loss(\theta,\z) - \nabla_\theta \calR(\theta)}{2}^2
= \Exp_{z \sim P}\norm{\z - \Exp[\z]}{2}^2 = \trace{\Sigma}.
\end{equation}
Therefore, the corresponding $\alpha(P, \loss)$ and $\beta(P, \loss)$ in Eq. \eqref{eqn:asp_variance_sg} are $0$ and $\trace{\Sigma}$.  Also, we note that $\tau_\ell = \tau_u = 1$ for the loss function $\loss(\theta,\z) = \frac{1}{2} \norm{\z - \theta}{2}^2$. By plugging these parameters to Theorem \ref{thm:clipped_sgd}, we get the desired clipping level $\lambda$ and the upper bound for $\norm{\theta^{N+1} - \theta^*}{2}$.

\subsection{Proof of Corollary \ref{cor:heavy_tailed_linear_regression}}

From Lemma 7 of \citet{prasad2018robust}, we have
\begin{equation}
\Exp[\nabla \loss(\theta, (x,y))] = \Sigma (\theta - \theta^*), 
\ \ \text{ and } \ \ 
\end{equation}
\begin{equation}
\norm{Cov(\nabla \loss(\theta, (x,y)))}{2}
\leq 2(C_4+1) \norm{\Sigma}{2}^2 \norm{\theta - \theta^*}{2}^2+ \sigma^2 \norm{\Sigma}{2},
\end{equation}
where $\Sigma = \Exp[xx^{\top}]$ is the covariance matrix of random variable $X$, $Cov(\nabla \loss(\theta, (x,y)))$ denotes the covariance matrix of $\nabla \loss(\theta, (x,y))$ and $C_4$ is the constant related to $4^{th}$ bounded moment defined in Eq. \eqref{eqn:asp_bounded_moment}. Since we have
$$
\Exp \norm{\nabla_\theta \loss(\theta,(x,y)) - \nabla_\theta \calR(\theta)}{2}^2 
= \trace{Cov(\nabla \loss(\theta, (x,y)))} \leq p \norm{Cov(\nabla \loss(\theta, (x,y)))}{2}.
$$
We obtain $\alpha(P, \loss) = 2p(C_4+1) \norm{\Sigma}{2}^2$ and $\beta(P, \loss)= p\sigma^2 \norm{\Sigma}{2}$ in Eq. \eqref{eqn:asp_variance_sg}. Also, by calculating the hessian matrix of population loss function $\calR(\theta)$,  we have
$\tau_\ell = \lambda_{min}(\Sigma)$ and $\tau_u = \norm{\Sigma}{2}$. Finally, by plugging these values to Theorem \ref{thm:clipped_sgd}, we got the desired bound and hyper-parameters.

\end{document}